\documentclass{article}

\PassOptionsToPackage{numbers, compress}{natbib}
\usepackage[preprint]{neurips_2026}

\makeatletter
\renewcommand{\@notice}{}
\makeatother

\usepackage[utf8]{inputenc}
\usepackage[T1]{fontenc}
\usepackage{microtype}
\usepackage{graphicx}
\usepackage{subcaption}
\usepackage{wrapfig}
\usepackage{algorithm}
\usepackage{algorithmic}
\usepackage{booktabs}
\usepackage{hyperref}
\usepackage{url}
\usepackage{xcolor}
\usepackage{amsmath}
\usepackage{amssymb}
\usepackage{mathtools}
\usepackage{amsthm}
\usepackage{tabularx}
\usepackage{thmtools}
\usepackage{thm-restate}
\usepackage[capitalize,noabbrev]{cleveref}
\usepackage[textsize=tiny]{todonotes}
\usepackage{multirow}

\title{AT-SKM-Net: An Accelerated Trainable Sampling Kaczmarz-Motzkin Framework for Linear Hard-Constraint Feasibility on Dynamic Graphs}

\author{Xiaochen Zhang \quad Haoyu Zhu \quad Yao Zhang \quad
  Qingchun Hou\thanks{Corresponding author: \href{mailto:houqingchun@zju.edu.cn}{\texttt{houqingchun@zju.edu.cn}}.} \\
  Zhejiang University \\
  \small\texttt{xiaochen.23@intl.zju.edu.cn} \quad \texttt{haoyu.22@intl.zju.edu.cn} \\
  \small\texttt{eezhangyao@zju.edu.cn} \quad \texttt{houqingchun@zju.edu.cn}}

\hypersetup{pdftitle={AT-SKM-Net: An Accelerated Trainable Sampling Kaczmarz-Motzkin Framework for Linear Hard-Constraint Feasibility on Dynamic Graphs},
  pdfauthor={Xiaochen Zhang; Haoyu Zhu; Yao Zhang; Qingchun Hou}}

\theoremstyle{plain}

\theoremstyle{definition}

\newtheorem{assumption}{Assumption}
\theoremstyle{remark}

\begin{document}

\maketitle

\begin{abstract}
Graph-structured optimization with linear constraints is fundamental to critical infrastructure but faces scalability limits due to massive strict hard constraints and high dimensionality. While recent projection-based methods such as Trainable Sampling Kaczmarz-Motzkin Net (T-SKM-Net) guarantee feasibility, they face high computational costs in dynamic environments by processing the entire constraint set and requiring expensive matrix factorizations. To bridge this gap, we propose the Accelerated Trainable-SKM (AT-SKM) Net framework. To concentrate computation on the active constraints and eliminate redundant calculations, we introduce a hybrid sampling strategy guided by a topology-aware heterogeneous GNN model. To efficiently handle topological shifts in graph-based constraints, we employ a Cholesky Update mechanism that theoretically reduces the equality projection complexity from $\mathcal{O}(N^3)$ to $\mathcal{O}(N^2)$ under low-rank perturbations. Experiments on random geometric graphs, N-1 Security-Constrained DC-OPF, and minimum-cost gas transport problem demonstrate that AT-SKM reduces iteration counts by up to 85\% and achieves 2.95$\times$-7.29$\times$ SKM layer speedups, while maintaining zero constraint violations. \end{abstract}

\section{Introduction}
Constrained optimization is essential for decision-making in a wide range of critical engineering systems \cite{kotary2021end}. In particular, graph-based formulations are particularly critical in infrastructure domains, such as power system dispatch and transportation network flow control \cite{intro_graph_optimization}. These applications typically face a dual challenge: on one hand, decisions must strictly satisfy hard constraints imposed by physical laws or operational requirements, where even minor violations are unacceptable; on the other hand, real-world environments are highly dynamic, for instance, sudden line failures in power grids cause instantaneous changes to the underlying graph topology \cite{intro_n-1_contingency}. This demands that solvers adapt to topological changes and produce feasible solutions within milliseconds, yet traditional optimization algorithms struggle to meet real-time control requirements when confronted with large-scale networks and frequent topological perturbations.

In recent years, neural network has been widely employed to approximate solutions for various optimization problems, due to its powerful function approximation capabilities \cite{intro_DN_for_Opt,utkarsh2025end}. However, standard neural networks cannot inherently guarantee the feasibility of their outputs. Existing approaches to address this limitation fall into two categories. Soft constraint methods suppress violations through penalty terms but lack theoretical safety guarantees, posing significant risks in safety-critical applications \cite{fioretto2020lagrangian, chamon2022constrained, kotary2024learning}. Hard constraint methods, such as differentiable optimization layers \cite{optnet, cvxpylayers, hardnet} and projection-based approaches \cite{intro_projection, deepopf}, can enforce feasibility but often require solving complex embedded optimization subproblems during inference, facing severe computational scalability bottlenecks that hinder their application to large-scale graph data with thousands of nodes \cite{intro_decisionrule}.

Recently, \citet{t-skm} proposed a Trainable Sampling Kaczmarz-Motzkin Net (T-SKM-Net) framework which successfully enables end-to-end training while guaranteeing linear hard constraint satisfaction by incorporating the SKM algorithm. However, T-SKM-Net relies on SVD to handle equality constraint projection, and topological changes necessitate expensive recomputation. Moreover, in practical optimization problems, typically only a few constraints are active. Consequently, blind uniform sampling wastes substantial computation budget on redundant constraints, then limiting convergence speed.

To address these challenges, we propose AT-SKM-Net, a trainable hard-constraint correction framework for dynamic graph-structured optimization. Instead of simply treating feasibility correction as a post-processing step, AT-SKM-Net aligns the correction procedure with two structural properties of graph-based constrained problems, low-rank topology-induced perturbations and sparse active constraints. Our main contributions are as follows:

\begin{itemize}
    \item We formulate a graph structure-aware trainable framework that preserves the hard-feasibility guarantee of projection-based SKM while solving the real-time scalability bottlenecks caused by large-scale constraints and topological perturbations.
    \item We accelerate the equality projection and inequality iteration bottleneck stages of SKM-based correction. For equality projection, we introduce Cholesky updates, reducing the adaptation complexity from $\mathcal{O}(N^3)$ to $\mathcal{O}(N^2)$ under low-rank topological perturbations. For inequality iteration, we introduce active-set-guided hybrid sampling to focus computation on likely binding constraints.
    \item We develop a heterogeneous graph neural network (HGNN) tailored to dynamic topologies to provide both warm-start candidates and active-set scores for AT-SKM-Net. Experiments demonstrate that AT-SKM-Net achieves substantial speedups while maintaining zero equality and inequality constraint violations.
\end{itemize}

\section{Related Work}

Existing work has made significant progress in improving the feasibility of solutions produced by neural solvers and traditional optimization methods. The literature relevant to our work can be broadly categorized into three distinct streams: neural methods for hard constraint satisfaction, acceleration strategies for the SKM algorithm, and applications of active-set prediction.

\subsection{Neural Methods for Hard Constraint Satisfaction}

\textbf{Differentiable Optimization Layer-Based Methods.} These approaches integrate optimization problems or iterative operations directly into the neural inference pipeline. While exact solvers like OptNet \cite{optnet} and CvxPyLayers \cite{cvxpylayers} enable differentiability for QP and general convex problems, their high computational complexity limits scalability. SCQPTH \cite{scqpth} accelerates solving via first-order ADMM, it remains restricted to convex QP formulations. Similarly, iterative projection methods such as LinSATNet \cite{linsatnet} and the entropy-regularized GLinSAT \cite{intro_DN_for_Opt} achieve constraint satisfaction but incur high computational costs due to the necessity of extensive iterations.

\textbf{Feasible Set Parameterization Methods.} These methods map the latent space onto the feasible region. CP \cite{frerix2020homogeneous} parameterizes solutions for linear inequalities but struggles with high dimensionality. While Rayen \cite{tordesillas2023rayen} and scaling-based approaches \cite{konstantinov2023new} extend support to quadratic and other convex constraints, they rely on a pre-identified strictly feasible interior point, rendering them ill-suited for constraints that vary dynamically with inputs. Safety decision rules \cite{intro_decisionrule} ensure feasibility via convex combinations of a solver and a safety network, but they require convex constraint sets and inherently compromise optimality.

\textbf{Correction and Projection Methods.} These methods correct infeasible predictions through post-processing or iterative updates. DC3 \cite{dc3} employs equality completion and inequality correction but strictly guarantees only equality satisfaction. ProjectNet \cite{projectnet} utilizes alternating projections but is limited to simple non-negative constraints. \citet{rashwan2025enforcing} enforce hard linear constraints via the Component-Averaged Dykstra algorithm. T-SKM-Net \cite{t-skm} achieves rigorous satisfaction of both equality and inequality constraints via null-space projection and SKM iterations. However, the use of SVD factorization and uniform constraint sampling may not fully utilize the structure of sparse active constraint sets and constraint matrix, limiting inference efficiency.

\subsection{Acceleration Methods for SKM}

Among efforts to accelerate Kaczmarz-type iterative methods, the Greedy Kaczmarz algorithm \cite{nutini2016convergence} has been proven to achieve the theoretically optimal convergence rate. However, each iteration requires scanning all constraint rows to find the maximum residual, resulting in prohibitively high per-step computational cost that hinders application to large-scale problems. While Randomized Block SKM \cite{zhang2023randomized} achieves faster convergence rates through the parallel iteration of multiple constraints, it is plagued by the difficulty of tuning parameters like sampling scale and block size. Probably Accelerated SKM (PASKM) \cite{paskm} incorporates Nesterov momentum into the SKM framework, leveraging historical gradient information to accelerate convergence toward the feasible region. Nevertheless, this method still employs uniform sampling and does not fundamentally address the inefficiency of constraint sampling. 

\subsection{Application of Active-Set Prediction}
\citet{xiang2016screening} demonstrated that active set prediction can effectively reduce both the computational time and memory usage required for solving Lasso problems. Similarly, \citet{misra2022learning} leveraged active set prediction to mitigate the complexity of constrained optimization; however, their approach necessitates solving reduced KKT systems for all identified candidate active sets, which can lead to substantial computational overhead if the prediction accuracy is insufficient. \citet{statistical_learning} propose ensemble-based active set prediction, predicting active constraints at the optimal solution based on input parameters. However, this method assumes the constraint system itself remains unchanged and cannot handle active set shifts caused by topological variations. \citet{gnn_warmstart_qp} proposed using GNN to predict active sets for warm-starting conventional solvers.

\section{Preliminaries}

Given a graph $\mathcal{G}$, consider the objective function $F$ and the feasibility of the state vector $\mathbf{x}\in\mathbb{R}^{n}$ subject to linear constraints:
\begin{equation}
\label{eq:problem_formulation}
\begin{aligned}
\min_{\mathbf{x}} \quad & F(\mathcal{G},\mathbf{x}) \\
\text{s.t.}\quad 
& A(\mathcal{G})\,\mathbf{x}\le \mathbf{b}, \, C(\mathcal{G})\,\mathbf{x}= \mathbf{d}
\end{aligned}
\end{equation}
where the constraint matrices $A(\mathcal{G})\in\mathbb{R}^{n_{\mathrm{ineq}}\times n}$ and $C(\mathcal{G})\in\mathbb{R}^{n_{\mathrm{eq}}\times n}$, with $\mathbf{b}\in\mathbb{R}^{n_{\mathrm{ineq}}}$ and $\mathbf{d}\in\mathbb{R}^{n_{\mathrm{eq}}}$ representing $n_{\mathrm{ineq}}$ inequality constraints and $n_{\mathrm{eq}}$ equality constraints, respectively.

The SKM with null space projection method first obtains an equality-feasible point by projecting an initial guess $\mathbf{x}_0$ onto the equality-constraint hyperplane via: \cite{t-skm}
\begin{equation}
\mathbf{x}_{\mathrm{eq}}=\mathbf{x}_0-C^\dagger(C\mathbf{x}_0-\mathbf{d}).
\end{equation}
Let $N\in\mathbb{R}^{n\times r}$ be a basis for $\mathrm{null}(C)$. Any equality-feasible iterate admits $\mathbf{x}=\mathbf{x}_{\mathrm{eq}}+N\boldsymbol{\omega}$ with $\boldsymbol{\omega}\in\mathbb{R}^{r}$. 

Let $\tilde{A} = AN$ with rows $\tilde{\mathbf{a}}_i^\top$, and $\tilde{\mathbf{b}} = \mathbf{b} - A\mathbf{x}_{\mathrm{eq}}$. Using a sampled constraint set $\mathcal{S}_k\subseteq\{1,\dots,n_{\mathrm{ineq}}\}$ and the hinge operator $[t]_+\triangleq \max\{t,0\}$, the SKM update on the null-space coefficients is:
\begin{equation}
\label{eq:nsskm_omega_form}
\begin{aligned}
i_k &\in \arg\max_{i\in\mathcal{S}_k}\; \Big[\tilde{\mathbf{a}}_i^\top \boldsymbol{\omega}_k-\tilde{b}_i\Big]_+,\\
\boldsymbol{\omega}_{k+1}
&=\boldsymbol{\omega}_k-\alpha\,
\frac{\Big[\tilde{\mathbf{a}}_{i_k}^\top \boldsymbol{\omega}_k-\tilde{b}_{i_k}\Big]_+}{\|\tilde{\mathbf{a}}_{i_k}\|_2^2}\,\tilde{\mathbf{a}}_{i_k},
\\
\mathbf{x}_{k+1}&=\mathbf{x}_{\mathrm{eq}}+N\boldsymbol{\omega}_{k+1}, \quad \boldsymbol{\omega}_0=\mathbf{0},
\end{aligned}
\end{equation}

where $\alpha\in(0,2)$ is the step-size. \cref{eq:nsskm_omega_form} enforces $C\mathbf{x}_k=\mathbf{d}$ while progressively reducing violations of the transformed inequalities.

\section{Accelerated Trainable-SKM Framework}

\begin{wrapfigure}{r}{0.48\textwidth}
  \centering
  \includegraphics[width=\linewidth]{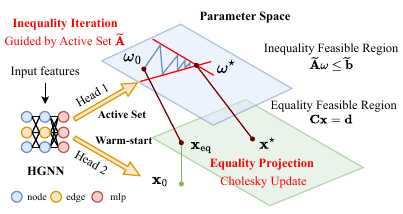}
  \caption{
    Flowchart of the AT-SKM-Net.
  }
  \label{fig:flowchart}
  \vspace{0.32in}
\end{wrapfigure}

We propose the Accelerated Trainable Sampling Kaczmarz-Motzkin Net (AT-SKM-Net) framework to resolve the conflict between strict feasibility guarantees and real-time scalability in dynamic environments. As shown in Figure~\ref{fig:flowchart}, given a graph instance with updated topology and constraint features, the HGNN first produces a warm-start candidate and active-set scores. The warm-start solution is then projected onto the equality-constraint manifold, and SKM iterations are performed in the null space to remove the remaining inequality violations. The complete inference procedure is summarized in Appendix~\ref{app:atskm_algorithm}. The framework aims to accelerate these two stages by using complementary mechanisms for equality projection and inequality iteration to achieve end-to-end speedup:

\textbf{Inequality SKM Iteration: Sampling Acceleration.} To overcome the inefficiency of uniform sampling, we introduce a hybrid strategy guided by Active-Set Prediction in Sec.~\ref{sec:active_set_sampling}. By leveraging topological priors to identify binding constraints, we effectively shrink the search space from the massive global set to the local active manifold. We provide the theoretical analysis in Sec.~\ref{sec:theoretical_analysis}.

\textbf{Equality Projection: Computing Acceleration.} To address the bottleneck of SVD refactorization under low-rank constraint matrix perturbations, we propose a Cholesky Update mechanism in Sec.~\ref{sec:cholesky_update}. This utilizes the low-rank nature of constraint matrix to theoretically reduce the equality projection complexity from $\mathcal{O}(N^3)$ to $\mathcal{O}(N^2)$.

These modules are powered by our proposed Heterogeneous GNN in Sec.~\ref{sec:hgnn}, which encodes the dynamic graph structure to warm-start the solver and guide the optimization process.

\subsection{Accelerated Sampling for SKM Iteration: Active-Set Prediction}
\label{sec:active_set_sampling}

We address the structural inefficiency of standard SKM methods arising from the mismatch between constraints sparsity and sampling uniformity. The geometry of the feasible region is typically characterized by a sparse set of active constraints, while the vast majority of constraints are redundant at the solution. However, standard SKM remains oblivious to this structure and samples uniformly from the total constraint set $m$. This kind of blind search leads to inefficient sampling, where the likelihood of selecting an informative constraint is significantly diminished by the large number of inactive constraints. 

To bridge this gap, we propose a neural-network-based hybrid sampling strategy that concentrates sampling probability on the active constraints, theoretically shifting the convergence dependence from the total number of constraints $m$ and the global Hoffman constant to the reduced active set size $|\hat{\mathcal{A}}|$ and the local condition number. 

\subsubsection{Hybrid Sampling Strategy via Predicted Active-sets}

Consider the linear inequality constraints over the polyhedron:
\begin{equation}
    P=\{\mathbf{x}\in\mathbb{R}^n \mid A\mathbf{x}\le \mathbf{b}\},
\end{equation}
where $A\in\mathbb{R}^{m\times n}$ and $\mathbf{a}_i^\top$ denotes the $i$-th row of $A$. Let $\mathbf{x}^*$ be the targeted feasible solution.

At $\mathbf{x}^*$, the active constraint set is defined as
\begin{equation}
    \mathcal{A}^*(\mathbf{x}^*) = \{ i \in \{1, \dots, m\} \mid \mathbf{a}_i^\top \mathbf{x}^* = b_i \}.
\end{equation}
In a neighborhood of $\mathbf{x}^*$, the feasible region is locally characterized by the affine manifold
$A_{\mathcal{A}^*}\mathbf{x}=\mathbf{b}_{\mathcal{A}^*}$ since other inequality constraints are still slack. 

To address the structural sampling inefficiency, we propose a hybrid sampling strategy guided by neural network predictions. This strategy aims to concentrate computational resources on constraints with high active probability, while retaining a component of uniform sampling to guarantee theoretical global convergence.

Let $\phi(\mathcal{G})$ be the neural network model. For an input optimization instance $\mathcal{G}$, the model outputs unnormalized logits $\mathbf{z} \in \mathbb{R}^m$. We transform the predicted active scores into a sampling distribution $\mathbf{p}$ via global normalization:
\begin{equation}
p_i = \frac{\sigma(z_i)}{\sum_{j=1}^m \sigma(z_j)}, \quad \forall i \in \{1, \dots, m\}.
\end{equation}

In each SKM iteration $k$, with a total batch size $\beta$ and a mixing ratio $\rho \in (0, 1]$, the sampled index set $\tau_k$ is constructed from two components:

\textbf{Prediction-Guided Exploitation ($\tau_{\text{learn}}$).}
    We select $\beta_1 = \lfloor \rho \cdot \beta \rfloor$ indices by weighted sampling based on distribution $\mathbf{p}$. This step mimics sampling from the predicted active set $\hat{\mathcal{A}}$, aiming to accelerate convergence by significantly reducing the effective sampling space dimension from $m$ to approximately $|\hat{\mathcal{A}}|$.
    
\textbf{Uniform Exploration ($\tau_{\text{unif}}$).}
    We uniformly sample $\beta_2 = \beta - \beta_1$ indices from the remaining constraint set. This step acts as a safeguard, preventing the omission of true active constraints due to prediction errors and ensuring the global robustness of the algorithm.

The final sampled constraint subset for iteration $k$ is $\tau_k = \tau_{\text{learn}} \cup \tau_{\text{unif}}$. This hybrid approach effectively leverages prior geometric information while overcoming the potential blind spots of purely prediction-based methods.

\subsection{Theoretical Analysis}
\label{sec:theoretical_analysis}

To quantify the acceleration achieved by dimension reduction, we introduce the following assumptions regarding the local geometry and the predictor's capability:

\begin{assumption}[Local Active-Set Stability]
    \label{ass:local_stability}
    There exists a neighborhood $\mathcal{U}$ of $\mathbf{x}^*$ such that for any iterate $\mathbf{x} \in \mathcal{U}$, the set of active constraints remains invariant at the projection, which means that all constraints inactive at $\mathbf{x}^*$ remain slack at $\Pi_P(\mathbf{x})$, and $\Pi_P(\mathbf{x})$ lies on the affine subspace defined by $\mathcal{A}^*(\mathbf{x}^*)$. That is, $\mathcal{A}(\Pi_P(\mathbf{x})) = \mathcal{A}^*(\mathbf{x}^*)$.
\end{assumption}

\begin{assumption}[Effective Active-Set Covering]
\label{ass:active_covering}
We assume the predictor identifies a candidate set $\hat{\mathcal{A}}$ that satisfies two conditions:
\begin{itemize}
    \item \textbf{Coverage}: It includes the true active set, i.e., $\mathcal{A}^* \subseteq \hat{\mathcal{A}}$.
    \item \textbf{Sparsity}: Its size is significantly smaller than the total number of constraints, i.e., $|\hat{\mathcal{A}}| \ll m$.
\end{itemize}
\end{assumption}

We provide a unified theoretical analysis of the acceleration mechanism. We first establish the global convergence guarantee of the algorithm, and then quantify its acceleration relative to standard SKM in the local linear regime.

\begin{restatable}[Global Convergence and Robustness]{theorem}{globalrobustness}
\label{thm:global_robustness_unified}
Provided that the mixing ratio satisfies $\rho < 1$, the proposed algorithm maintains global linear convergence in expectation toward the feasible set $P$.
\end{restatable}
\begin{proof}[\textbf{Proof Sketch}]
The proof relies on the positive sampling probabilities. Since $\rho < 1$, every constraint retains a strictly positive probability of being sampled. This satisfies the sufficient condition for global linear convergence established in standard randomized Kaczmarz theory \cite{haddock2021SKM, de2017sampling}, ensuring robustness against prediction errors. Complete proof in Appendix~\ref{app:active_sampling_proof}.
\end{proof}

\begin{restatable}[Efficiency Gain via Dimension Reduction]{theorem}{efficiencygain}
\label{thm:comparative_acceleration_corrected}
Consider the iterate in the local linear regime from Assumption~\ref{ass:local_stability} with row-normalized constraints. Comparing uniform sampling over $\{1, \dots, m\}$ and our proposed uniform sampling over the predicted set $\hat{\mathcal{A}}$, under Assumption~\ref{ass:active_covering}, the expected one-step error contraction rates are:

\begin{itemize}
    \item \textbf{Standard SKM (Uniform on $m$):}
    \begin{equation}
        \mathbb{E}[\|\mathbf{e}_{k+1}\|^2] \le \left( 1 - \frac{1}{m \cdot \mathcal{H}^2(A)} \right) \|\mathbf{e}_k\|^2
    \end{equation}
    \item \textbf{Accelerated SKM (Uniform on $\hat{\mathcal{A}}$):}
    \begin{equation}
        \mathbb{E}[\|\mathbf{e}_{k+1}\|^2] \le \left( 1 - \frac{\sigma_{\min}^2(A_{\mathcal{A}^*})}{|\hat{\mathcal{A}}|} \right) \|\mathbf{e}_k\|^2
    \end{equation}
\end{itemize}
where $\mathcal{H}(A)$ is the global Hoffman constant, and $\sigma_{\min}(A_{\mathcal{A}^*})$ denotes the smallest non-zero singular value of the true active submatrix.
\end{restatable}

\begin{proof}[\textbf{Proof Sketch}]
The acceleration stems from the error contraction behavior in the local linear regime by Assumption \ref{ass:local_stability}. In neighborhood $\mathcal{U}$, the hinge loss for any inactive constraint has $[a_i^\top x - b_i]_+ = 0$. Consequently, the summation term in the expected update collapses from the full constraint set to the active set $\mathcal{A}^*$. Complete proof in Appendix~\ref{app:active_sampling_proof}.
\end{proof}

This theorem shows the acceleration through two complementary mechanisms: sampling efficiency and geometric regularization. First, the reduction in the denominator from $m$ to $|\hat{\mathcal{A}}|$ reflects the compression of the sampling space, yielding significant acceleration whenever $|\hat{\mathcal{A}}| \ll m$. Second, the convergence rate is governed by the singular value $\sigma_{\min}$ of the true active set $\mathcal{A}^*$ rather than the Hoffman constant $\mathcal{H}$, which accounts for the worst-case conditioning among all submatrices \cite{error_bounds_hoffman}. 

\subsection{Accelerated Computing for Equality Projection: Cholesky Update}
\label{sec:cholesky_update}

In T-SKM-Net, the SKM layer first projects the neural prediction $\mathbf{x}_0$ onto the equality manifold  $\mathbf{C}\mathbf{x}=\mathbf{d}$ and then performs inequality iterations in the null space of C. This equality projection is  essential for preserving feasibility, but it becomes a major bottleneck in dynamic graph problems. When the graph topology changes, the equality matrix C also changes, and the SVD-based projection  used in T-SKM-Net has to be recomputed under an expensive $\mathcal{O}(n^3)$ computational cost for each perturbed instance~\cite{graph_laplacian_low_rank}. 

However, mild topological changes, such as line outages or reconnections, typically induce localized modifications in the adjacency matrix or the graph Laplacian. Motivated by this observation, we propose an incremental Cholesky update framework that exploits such structured changes to accelerate the equality-constraint projection phase.

Given a row full-rank matrix $\mathbf{C} \in \mathbb{R}^{m \times n}$, a rank-1 perturbation defined by $\tilde{\mathbf{C}} = \mathbf{C} + \mathbf{u}\mathbf{v}^\top$ induces a rank-2 modification in its corresponding Gram matrix $\mathbf{G} = \mathbf{C}\mathbf{C}^\top$. To handle this efficiently, we bypass computationally expensive SVD re-computations in favor of a generic incremental Cholesky update strategy. By sequentially applying one rank-1 update and one rank-1 downdate to the original $\mathbf{G}=\mathbf{L}\mathbf{L}^\top$, we can directly evolve the triangular factor $\tilde{\mathbf{L}}$ to adapt to the new topology without reconstructing the matrix from scratch.

Based on the updated factor $\tilde{\mathbf{L}}$, we transform the explicit matrix inversion into two efficient triangular matrix solves. For the constraint residual $\mathbf{r} = \tilde{\mathbf{C}}\mathbf{x}_0 - \mathbf{d}$, the projected solution can be expressed as:
\begin{equation}
\tilde{\mathbf{x}}_\text{eq} = \mathbf{x}_0 - \tilde{\mathbf{C}}^\top (\tilde{\mathbf{L}}^{-\top} (\tilde{\mathbf{L}}^{-1} \mathbf{r}))
\end{equation}

The null-space representation used by the subsequent SKM iterations is updated consistently with the perturbed equality constraints. Instead of reconstructing the whole null-space basis from scratch, we reuse the previous basis and correct only the low-dimensional subspace affected by the perturbation~\cite{nocedal2006numerical}. Intuitively, a low-rank change in $\mathbf{C}$ only alters a small number of directions in the null space, while the remaining directions can be retained after an orthogonal alignment. The explicit construction of this basis update is provided in Appendix~\ref{app:cholesky_imple}.

\paragraph{Computational Complexity.} While the SVD approach necessitates a full re-factorization with cubic complexity $\mathcal{O}(N^3)$, our proposed framework significantly reduces the computational burden by exploiting the low-rank perturbations of constraint matrix. By replacing global decomposition with localized Cholesky rank-1 updates and efficient basis rotations, we effectively decouple the high-dimensional problem into sequential vector operations and triangular solves. This strategy lowers the overall complexity to $\mathcal{O}(N^2)$, offering a theoretical speedup of order $N$ compared to standard methods. Description with more details is in Appendix~\ref{app:choleskycomplexity}.

\subsection{Topology-Aware Heterogeneous Graph Neural Network for Active-Set Prediction}
\label{sec:hgnn}

The hybrid sampling strategy proposed in \cref{sec:active_set_sampling} relies on precise estimation of active probabilities for individual constraints, which poses two key challenges: capturing the heterogeneity of constraints, and generalizing across varying topological scenarios. To address such challenges, we propose a heterogeneous graph neural network (HGNN) architecture that aligns with the structure of constraints and topology of graph-constrained optimization problems.

\subsubsection{Heterogeneous Graph Representation}

In graph-based optimization problems, constraints are naturally partitioned into node-level constraints and edge-level constraints. Homogeneous GNNs treat edges as message-passing channels, which tends to smooth edge-specific information into node representations \cite{hgnn-opf}. This smoothing effect hinders the extraction of distinct active probabilities for edge-based couplings.

To address this, we adopt a heterogeneous bipartite graph architecture where both original nodes and edges are modeled as distinct computational entities. This design establishes a structural correspondence between neural network embeddings and the optimization constraint set. Let $\mathbf{h}_u^{(l)}$ and $\mathbf{h}_{e}^{(l)}$ denote the latent representations of node $u$ and edge $e=(u,v)$ at layer $l$. The heterogeneous message passing updates are:
\begin{equation}
\mathbf{h}_{e}^{(l+1)} = \phi_{\text{edge}} \left( \mathbf{h}_{e}^{(l)}, \mathbf{h}_{u}^{(l)}, \mathbf{h}_{v}^{(l)} \right), \qquad
\mathbf{h}_{u}^{(l+1)} = \phi_{\text{node}} \left( \mathbf{h}_{u}^{(l)}, \sum_{e \in \mathcal{N}(u)} \mathbf{h}_{e}^{(l+1)} \right),
\label{eq:hetero_update}
\end{equation}
where $\phi_{\text{edge}}$ and $\phi_{\text{node}}$ are nonlinear transition functions. This formulation preserves the representation $\mathbf{h}_e$ of each coupling constraint independently throughout the propagation depth, rather than absorbing it into nodal states.

\subsubsection{Static-Dynamic Feature Separation}

When the graph topology changes, the neural model needs to generalize across these perturbations. We observe that problem inputs can be decomposed into two categories: static geometric parameters $\mathbf{z}_{\text{stat}}$ that define the constraint coefficient matrix $A$, and dynamic instance conditions $\mathbf{z}_{\text{dyn}}$ that define the boundary vector $\mathbf{b}$.

To prevent static structural information from being diluted by dynamic variations during deep message passing, we employ a structure-aware normalization mechanism. The mechanism first concatenates the dynamic hidden features $\mathbf{h}^{(l)}$ with static physical features $\mathbf{z}_{\text{stat}}$, then generates the normalization parameters through graph convolution:
\begin{equation}
\mathbf{h}^{(l+1)} = \gamma(\mathbf{h}^{(l)}, \mathbf{z}_{\text{stat}}) \odot \frac{\mathbf{h}^{(l)} - \mu}{\sigma} + \beta(\mathbf{h}^{(l)}, \mathbf{z}_{\text{stat}}),
\label{eq:topology_injection}
\end{equation}
where the affine parameters $\gamma$ and $\beta$ are generated from the concatenation $[\mathbf{h}^{(l)} \| \mathbf{z}_{\text{stat}}]$. This formulation is related to adaptive normalization method proposed by \cite{GRANOLA}, but instead of injecting random noise, we condition on the global static topology of the graph. This helps ensure that predictions remain consistent with constraints of $\mathcal{G}$ when its topology changes.

\subsubsection{Output Heads}

The HGNN serves as a generator for the accelerated solver. Its final embeddings are fed into two output heads. A regression head predicts the primal variable candidate $\hat{\mathbf{x}}_0$ to provide a warm start for the SKM layer. A classification head generates active scores for both node and edge constraints. The sampling distribution $\mathbf{p}$ required by the hybrid strategy is obtained via global normalization:
\begin{equation}
p_k = \frac{\sigma(\text{MLP}(\mathbf{h}_k))}{\sum_{j \in \mathcal{V} \cup \mathcal{E}} \sigma(\text{MLP}(\mathbf{h}_j))},
\label{eq:prob_output}
\end{equation}
where $k$ indexes over all node and edge constraints.

\section{Experiments}
\label{sec:exp}

We evaluate the proposed AT-SKM method on random geometric graphs and N-1 security-constrained DC optimal power flow problem, and provide an additional validation on the minimum-cost gas transport problem in Appendix~\ref{appendix:min_cost}. All experiments are conducted on an Apple M4 chip.

\subsection{Equality Projection on Random Geometric Graphs}

We validate the asymptotic computational complexity of the proposed Cholesky update method for equality-constraint projection using random geometric graphs (RGGs) of varying sizes. In an RGG, nodes are uniformly distributed within the unit square, with edges connecting pairs whose Euclidean distance falls below a threshold $r$; this local connectivity structure closely resembles real-world networks where physical proximity constrains connections. 

\begin{wrapfigure}{r}{0.48\textwidth}
  \vspace{-0.12in}
  \centering
  \includegraphics[width=\linewidth]{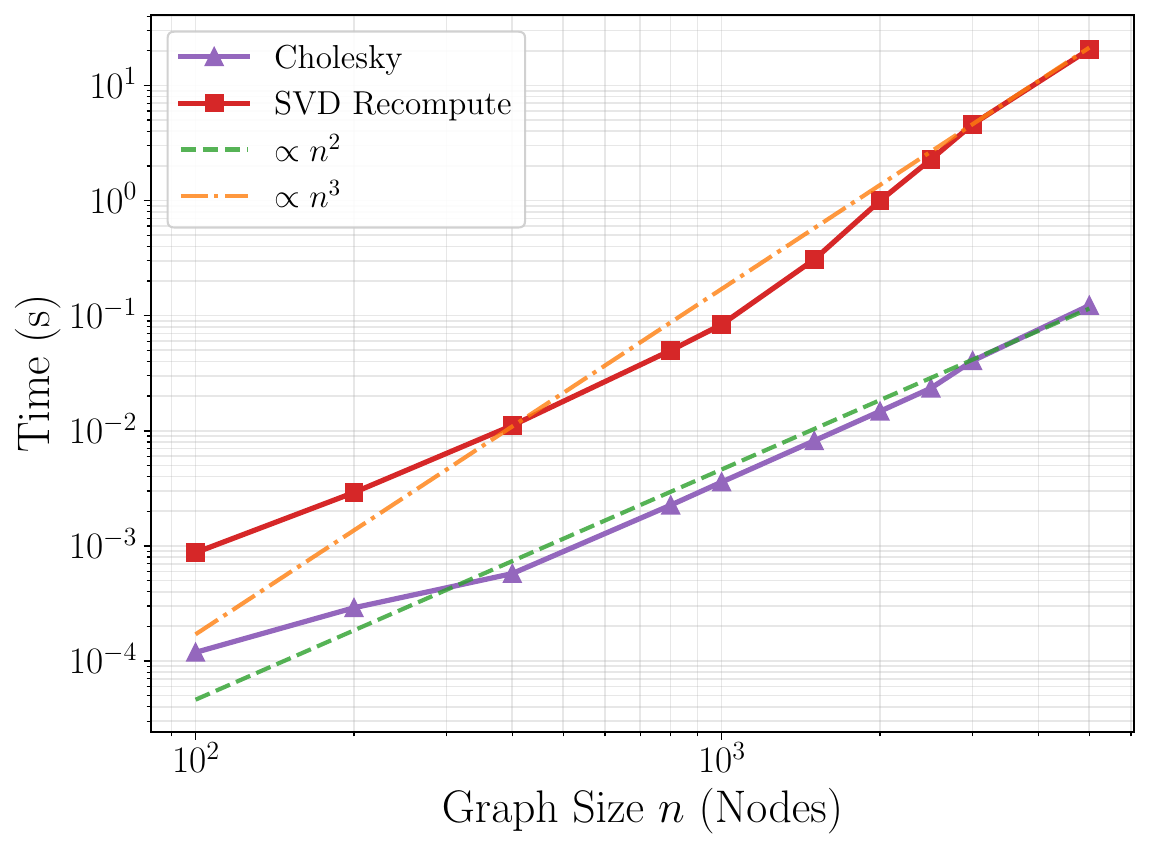}
  \caption{
    Computational complexity comparison between SVD recomputation and Cholesky update on RGGs. SVD exhibits $\mathcal{O}(n^3)$ scaling while Cholesky update achieves $\mathcal{O}(n^2)$.
  }
  \label{fig:Complexity_comparision}
  \vspace{-0.22in}
\end{wrapfigure}

\cref{fig:Complexity_comparision} compares the computation time of the Cholesky Update method, which incrementally updates the initial factorization, against SVD recomputation that performs a full decomposition for each topology change. On a log-log scale, the SVD computation time closely follows the $\mathcal{O}(N^3)$ reference curve as the number of nodes increases. Meanwhile, the empirical runtime of our proposed method aligns well with $\mathcal{O}(N^2)$. Notably, when the graph size reaches approximately 2000 nodes, SVD requires seconds to complete, which is prohibitive for safety-critical real-time applications such as power system dispatch. In contrast, the Cholesky update maintains computation times on the order of $10^{-2}$ seconds, demonstrating a substantial efficiency advantage. A detailed breakdown of these computational overheads is provided in Appendix~\ref{app:rgg_analysis}.

\subsection{N-1 Security-Constrained DC Optimal Power Flow}

We evaluate the performance of our proposed method on the N-1 Security-Constrained DC Optimal Power Flow (N-1 SC DC-OPF) problem. This problem stands as a fundamental challenge in critical power system infrastructure and represents a distinct class of graph-based constrained optimization characterized by strict safety requirements and dynamic network topologies \cite{dcopf_importance}.

The primary objective is to minimize costs subject to physical constraints, including power balance equations, generator output limits, and transmission line thermal limits. The rigorous N-1 security criterion mandates that the system must maintain feasibility following the loss of a single transmission component~\cite{sc-dcopf}. The detailed mathematical formulation is provided in Appendix~\ref{appendix:scopf}.

The experiment is based on IEEE 57-bus, 118-bus, and 300-bus systems. While we consider all feasible N-1 contingencies, we include only a subset of these scenarios in the training set to assess the model's adaptability and generalization to unseen topologies. The testing set evaluates the solver's performance in terms of cost optimization and constraint satisfaction across both seen and unseen topologies. The ground truth solutions for all instances are generated using the PyPower library \cite{pypower}.

We benchmark our method against: (1) \textbf{PyPower}: A traditional interior-point solver serving as the ground truth; (2) \textbf{NN} / \textbf{HGNN}: End-to-end learning baselines without explicit constraint enforcement; (3) \textbf{DC3}: A differentiable optimization method guaranteeing equality constraints via projection; and (4) \textbf{T-SKM-Net}: The original Trainable SKM framework. To decouple the contributions of our core modules, we evaluate three variants: \textbf{AT-SKM-A} (Active-set prediction only), \textbf{AT-SKM-C} (Cholesky update only), and \textbf{AT-SKM-AC} (Full).

\begin{table*}[tb]
  \caption{Performance comparison across different methods. The Optimality Gap and Constraint Violations (tol=$10^{-5}$) are formatted as ``Seen / Unseen''. The ``Iter. Num.'' column reports the mean (max) number of inequality iterations. The SKM-related metrics are reported on unseen scenarios to demonstrate robust acceleration.}
  \label{tab:dcopf_comp}
  \centering
  \begin{scriptsize}
    \setlength{\tabcolsep}{2.5pt}
    \begin{sc}
      \begin{tabularx}{\textwidth}{lccccccc}
        \toprule
        Method &
        \shortstack[c]{Tot. Time (ms)} &
        \shortstack[c]{Iter. Num.} &
        \shortstack[c]{SKM Time} &
        \shortstack[c]{SKM Speedup} &
        \shortstack[c]{Opt. Gap (\%)} &
        \shortstack[c]{Max Eq. Vio.} &
        \shortstack[c]{Max Ineq. Vio.} \\
        \midrule
        \multicolumn{8}{c}{\textbf{IEEE 57-Bus System: $n=64$, $n_{\text{eq}}=58$, $n_{\text{ineq}}=174$}}\\
        \midrule
        PyPower  & 14.785  & - & - & - & 0 / 0 & 0 / 0 & 0 / 0\\
        NN       & 0.048 & - & - & - & 0.173 / 5.379 & 24.77 / 53.07 & 0.101 / 192.71\\
        HGNN & 4.735 & - & - & - & 0.071 / 0.061 & 17.52 / 5.163 & 0.061 / 0.041 \\
        DC3  & 6.278 & - & - & - & 0.013 / 0.013 & 0 / 0 & 0.067 / 0.051\\
        T-SKM-Net &  5.831 & 6.93 (148) & 1.096 & 1$\times$ & 0.008 / 0.008 & 0 / 0 & 0 / 0\\
        AT-SKM-A & 5.202 & \textbf{1.42 (11)} &  0.467 & 2.35$\times$  & 0.008 / 0.008 & 0 / 0 & 0 / 0\\
        AT-SKM-C & 5.353 & 6.93 (148) & 0.618 & 1.77$\times$ & 0.008 / 0.008 & 0 / 0 & 0 / 0\\
        AT-SKM-AC & 5.009 & \textbf{1.42 (11)} & \textbf{0.274} & \textbf{4.00$\times$} & 0.008 / 0.008 & 0 / 0 & 0 / 0\\
        \midrule
        \multicolumn{8}{c}{\textbf{IEEE 118-Bus System: $n=172$, $n_{\text{eq}}=119$, $n_{\text{ineq}}=480$}}\\
        \midrule
        PyPower  & 31.557  & - & - & - & 0 / 0 & 0 / 0 & 0 / 0\\
        NN       & 0.057 & - & - & - & 0.142 / 1.481 & 72.59 / 93.21 & 0.031 / 664.11\\
        HGNN & 6.938 & - & - & - & 0.017 / 0.016 & 61.06 / 67.36 & 0.010 / 0.047 \\ DC3  & 8.040 & - & - & - & 0.019 / 0.019 & 0 / 0 & 0.011 / 0.009 \\
        T-SKM-Net & 10.251 & 34.60 (141) & 3.313 & 1$\times$ & 0.026 / 0.025 & 0 / 0 & 0 / 0\\
        AT-SKM-A & 9.349 & \textbf{15.68 (57)} &  2.411 & 1.37$\times$  & 0.026 / 0.025 & 0 / 0 & 0 / 0\\
        AT-SKM-C & 8.942 & 34.60 (141) & 2.004 & 1.65$\times$ & 0.026 / 0.025 & 0 / 0 & 0 / 0\\
        AT-SKM-AC & 8.060 & \textbf{15.68 (57)} & \textbf{1.122} & \textbf{2.95$\times$} & 0.026 / 0.025 & 0 / 0 & 0 / 0\\
        \midrule
        \multicolumn{8}{c}{\textbf{IEEE 300-Bus System: $n=369$,  $n_{\text{eq}}=301$, $n_{\text{ineq}}=960$}}\\
        \midrule
        PyPower  & 64.398 & - & - & - & 0 / 0 & 0 / 0 & 0 / 0\\
        NN     & 0.078 & - & - & - & 0.057 / 0.367 & 2.00e+3 / 9.44e+3 & 1.53e+3 / 8.52e+3\\
        HGNN & 9.390 & - & - & - & 0.068 / 0.068 & 64.61 / 62.36 & 537.66 / 0.086\\
        DC3  & 16.138 & - & - & - & 0.102 / 0.102 & 0 / 0 & 0.471 / 0.465\\
        T-SKM-Net & 18.382 & 24.72 (285) & 8.992 & 1$\times$ & 0.108 / 0.108 & 0 / 0 & 0 / 0\\
        AT-SKM-A & 17.264 & \textbf{3.68 (9)} & 7.874 & 1.14$\times$ & 0.108 / 0.108 & 0 / 0 & 0 / 0\\
        AT-SKM-C & 12.179 & 24.72 (285) & 2.789 & 3.22$\times$ & 0.110 / 0.110 & 0 / 0 & 0 / 0\\
        AT-SKM-AC & 10.623 & \textbf{3.69 (9)} & \textbf{1.233} & \textbf{7.29$\times$} & 0.110 / 0.110 & 0 / 0 & 0 / 0\\
        \bottomrule
      \end{tabularx}
    \end{sc}
  \end{scriptsize}
  \vskip -0.25in
\end{table*}

As shown in \cref{tab:dcopf_comp}, our proposed method achieves significant speedups over the T-SKM-Net baseline across all system sizes, while strictly satisfying all equality and inequality constraints without compromising solution quality. Regarding solution quality, although unconstrained baselines NN and HGNN achieve small optimality gaps, they exhibit varying degrees of constraint violations, rendering them unsafe for real-world operations. Moreover, the NN shows poor generalization on unseen topologies. While DC3 satisfies equality constraints, it fails to eliminate inequality violations.

We further analyze the isolated impacts of active-set prediction and Cholesky updates:

\textbf{Active-Set Prediction:} AT-SKM-A reduces the average iteration count by 80\%, 55\%, and 85\% on the 57, 118, and 300-bus systems, respectively. This confirms that our hybrid sampling strategy effectively concentrates computational resources on the sparse subset of active constraints.

\textbf{Cholesky Update:} AT-SKM-C significantly accelerates the projection phase. In the 300-bus system, although the results indicate that the Cholesky update does not affect the inequality iteration rate, AT-SKM-C reduces the SKM runtime from 8.99ms to 2.79ms ($3.22\times$ speedup), whereas AT-SKM-A reduces iterations by 85\% but only yields a $1.14\times$ speedup. This reveals that in large-scale systems, the computational bottleneck shifts from the iterative search to linear algebraic operations, highlighting the $\mathcal{O}(n^2)$ advantage of our Cholesky update.

Interestingly, we observe a synergistic effect over these two methods. The total speedup of AT-SKM-AC approximates or even exceeds the product of the individual speedups. For instance, on the 57-bus system, $S_A (2.35\times) \times S_C (1.77\times) \approx S_{AC} (4.00\times)$, and on the 300-bus system, the combined speedup ($7.29\times$) significantly surpasses the product of individual gains $S_A(1.14\times) \times S_C(3.22\times) \ll S_{AC}(7.29\times)$, suggesting a super-linear efficiency improvement.

Further experiments are provided in the appendix: hybrid sampling sensitivity (\ref{appendix:hybrid_sampling_sensitivity}), detailed runtime breakdown (\ref{appendix:time_breakdown_dcopf}), high-load robustness (\ref{appendix:high_load_robustness}), active-set prediction quality (\ref{appendix:active_set_quality}), active-set perturbation (\ref{appendix:active_set_perturbation}), and transfer learning with pseudo-label active-set training (\ref{appendix:transfer_pseudo_label}).

\section{Conclusion}
  \vskip -0.1in
In this paper, we propose the AT-SKM framework to effectively address the trade-off between strict feasibility and real-time scalability in graph-structured optimization. By combining neural network active-set predictions with efficient Cholesky incremental updates, AT-SKM-Net overcomes the computational bottlenecks that limit the original projection-based T-SKM-Net in dynamic environments.

\bibliography{Main}
\bibliographystyle{unsrtnat}

\newpage
\appendix
\section{AT-SKM-Net Algorithm Flow and Theoretical Proofs}

\subsection{AT-SKM-Net Algorithm Flow}
\label{app:atskm_algorithm}

\begin{algorithm}[H]
  \caption{Pseudo-code of the proposed AT-SKM-Net Framework.}
  \label{alg:atskm}
  \footnotesize
  \begin{algorithmic}[1]
    \REQUIRE Constraint data $(A,b,C,d)$, predictor $\phi$, sampling batch size $\beta$, mixing ratio $\rho$, maximum iterations $K$
    \ENSURE Feasible solution $x_K$
    \STATE $(x_0,s) \leftarrow \phi(A,b,C,d)$
    \STATE $p_i \leftarrow \sigma(s_i)/\sum_{j=1}^{m}\sigma(s_j),\ i=1,\ldots,m$
    \STATE \textbf{Equality projection stage.}
    \STATE $L \leftarrow \textsc{CholeskyUpdate}(C)$ such that $LL^\top = CC^\top$
    \STATE Solve $LL^\top y=Cx_0-d$ and set $x_{\mathrm{eq}}=x_0-C^\top y$
    \STATE $N \leftarrow \textsc{NullSpaceUpdate}(C,L),\quad CN=0$
    \STATE $\widetilde{A}\leftarrow AN,\quad \widetilde{b}\leftarrow b-Ax_{\mathrm{eq}},\quad \omega_0\leftarrow 0$
    \STATE \textbf{Inequality iteration stage.}
    \FOR{$k=0,1,\ldots,K-1$}
      \STATE $\beta_{\mathrm{act}}\leftarrow \lfloor\rho\beta\rfloor,\quad \beta_{\mathrm{unif}}\leftarrow \beta-\beta_{\mathrm{act}}$
      \STATE Sample $\tau_{\mathrm{act}}$ from $\{1,\ldots,m\}$ according to $p$ with size $\beta_{\mathrm{act}}$
      \STATE Uniformly sample $\tau_{\mathrm{unif}}$ from $\{1,\ldots,m\}$ with size $\beta_{\mathrm{unif}}$
      \STATE $\tau_k\leftarrow \tau_{\mathrm{act}}\cup\tau_{\mathrm{unif}}$
      \STATE $i_k\leftarrow\arg\max_{i\in\tau_k}[\widetilde{a}_i^\top\omega_k-\widetilde{b}_i]_+$
      \IF{$\widetilde{a}_{i_k}^\top\omega_k>\widetilde{b}_{i_k}$}
        \STATE $\omega_{k+1} =\omega_k-\alpha\frac{\widetilde a_{i_k}^\top \omega_k-\widetilde b_{i_k}}{\|\widetilde a_{i_k}\|_2^2}\widetilde a_{i_k},\qquad \alpha\in(0,2).$
      \ELSE
        \STATE $\omega_{k+1}\leftarrow\omega_k$
      \ENDIF
    \ENDFOR
    \STATE $x_K\leftarrow x_{\mathrm{eq}}+N\omega_K$
    \STATE \RETURN $x_K$
  \end{algorithmic}
\end{algorithm}

\subsection{Accelerated Sampling by Active-Set Prediction}
\label{app:active_sampling_proof}

\globalrobustness*

\begin{proof}
The AT-SKM algorithm iterates on the null-space coefficients $\omega$. Let $\Omega = \{\omega \in \mathbb{R}^r \mid \tilde{A}\omega \le \tilde{b}\}$ be the feasible set in the null-space representation, where $\tilde{A} = AN$ and $\tilde{b} = b - Ax_{eq}$. Let $\omega^*$ be the projection of the current iterate $\omega_k$ onto $\Omega$. We define the error vector as $e_k = \omega_k - \omega^*$.

Consider the $k$-th iteration. Let $g_j(\omega_k) = \frac{[\tilde{a}_j^\top \omega_k - \tilde{b}_j]_+^2}{\|\tilde{a}_j\|^2}=[\tilde{a}_j^\top \omega_k - \tilde{b}_j]_+^2$ denote the squared normalized violation of the row-normalized $j$-th constraint. The algorithm selects the active constraint index $i_k$ from a sampled batch $\mathcal{S}_k$ using the greedy criterion:
\begin{equation}
i_k = \mathop{\arg\max}_{j \in \mathcal{S}_k} g_j(\omega_k).
\end{equation}

The standard Kaczmarz update with step size $\alpha \in (0, 2)$ yields the following error contraction recurrence \cite{de2017sampling}:
\begin{equation}
    \|e_{k+1}\|^2 \le \|e_k\|^2 - \alpha(2-\alpha) g_{i_k}(\omega_k).
\end{equation}

The hybrid sampling strategy ensures that the batch $\mathcal{S}_k$ contains a subset $\tau_\text{unif}$ of size $\beta_\text{unif} \ge 1$, which is sampled uniformly from all $m$ constraints, which is guaranteed by $\rho < 1$. The maximum violation in the full batch is bounded below by the maximum violation in the uniform subset, which is further bounded below by the average violation in that subset:
\begin{equation}
    g_{i_k}(\omega_k) = \max_{j \in \mathcal{S}_k} g_j(\omega_k) \ge \max_{j \in \tau_\text{unif}} g_j(\omega_k) \ge \frac{1}{\beta_\text{unif}} \sum_{j \in \tau_\text{unif}} g_j(\omega_k).
\end{equation}

We take the expectation with respect to the random sampling of the batch. Since $\tau_\text{unif}$ is drawn uniformly from $\{1, \dots, m\}$, the probability that any specific constraint $l$ is included in $\tau_\text{unif}$ is $\beta_\text{unif}/m$. By the linearity of expectation:

\begin{equation}
    \begin{aligned}
\mathbb{E}[g_{i_k}(\omega_k)] &\ge \mathbb{E}\left[ \frac{1}{\beta_\text{unif}} \sum_{j \in \tau_\text{unif}} g_j(\omega_k) \right] \\
&= \frac{1}{\beta_\text{unif}} \sum_{l=1}^m \frac{\beta_\text{unif}}{m} g_l(\omega_k) \\
&= \frac{1}{m} \sum_{l=1}^m g_l(\omega_k).
\end{aligned}
\end{equation}

The summation term corresponds to the squared norm of the residual vector $\|[\tilde{A}\omega_k - \tilde{b}]_+\|^2$. By the global Hoffman error bound theorem, there exists a constant $\mathcal{H}(\tilde{A})$ such that:

\begin{equation}
    \sum_{l=1}^m g_l(\omega_k) = \|[\tilde{A}\omega_k - \tilde{b}]_+\|^2 \ge \frac{1}{\mathcal{H}^2(\tilde{A})} \|e_k\|^2.
\end{equation}

Substituting this back into the expectation of the error recurrence:

\begin{equation}
    \mathbb{E}[\|e_{k+1}\|^2] \le \left( 1 - \frac{\alpha(2-\alpha)}{m \cdot \mathcal{H}^2(\tilde{A})} \right) \|e_k\|^2=\left( 1 - \frac{1}{m \cdot \mathcal{H}^2(\tilde{A})} \right) \|e_k\|^2.
\end{equation}

Under $\alpha = 1 \in (0,2) $ and $\mathcal{H}(\tilde{A}) < \infty$, the algorithm converges linearly in expectation. 

\end{proof}

\efficiencygain*

\begin{proof}
Let $\mathbf{e}_k = \mathbf{x}_k - \mathbf{x}^*$. Given row normalization, the Kaczmarz update with hinge loss implies $\|\mathbf{e}_{k+1}\|^2 = \|\mathbf{e}_k\|^2 - [ \mathbf{a}_i^\top \mathbf{x}_k - b_i ]_+^2$.

Taking the expectation over the sampled index $i$:
\begin{equation}
    \mathbb{E}[\|\mathbf{e}_{k+1}\|^2] = \|\mathbf{e}_k\|^2 - \sum_{i \in \mathcal{S}} p_i [ \mathbf{a}_i^\top \mathbf{x}_k - b_i ]_+^2.
\end{equation}

For \textbf{Standard SKM}, $\mathcal{S}=\{1, \dots, m\}$ and $p_i = 1/m$. Using the global Hoffman bound on the residual vector $[A\mathbf{x}_k - \mathbf{b}]_+$, we have $\frac{1}{m}\|[A\mathbf{x}_k - \mathbf{b}]_+\|^2 \ge \frac{1}{m \cdot \mathcal{H}^2(A)} \|\mathbf{e}_k\|^2$ \cite{levenStandardSKM}.

For \textbf{Accelerated SKM}, $\mathcal{S}=\hat{\mathcal{A}}$ and $p_i = 1/|\hat{\mathcal{A}}|$. 
Crucially, under Assumption~\ref{ass:local_stability}, for any inactive constraint $i \notin \mathcal{A}^*$, we have $\mathbf{a}_i^\top \mathbf{x}_k < b_i$, implying $[ \mathbf{a}_i^\top \mathbf{x}_k - b_i ]_+ = 0$. 
Thus, the summation vanishes for all inactive constraints and retains only the contribution from the true active set $\mathcal{A}^*$. For $i \in \mathcal{A}^*$, since $\mathbf{a}_i^\top \mathbf{x}^* = b_i$, the term simplifies to $[ \mathbf{a}_i^\top \mathbf{x}_k - b_i ]_+^2 = \|\mathbf{a}_i^\top \mathbf{e}_k\|^2$:
\begin{equation}
\begin{aligned}
    \sum_{i \in \hat{\mathcal{A}}} \frac{1}{|\hat{\mathcal{A}}|} [ \mathbf{a}_i^\top \mathbf{x}_k - b_i ]_+^2 &= \frac{1}{|\hat{\mathcal{A}}|} \sum_{i \in \mathcal{A}^*} \|\mathbf{a}_i^\top \mathbf{e}_k\|^2 \\
    &= \frac{1}{|\hat{\mathcal{A}}|} \|A_{\mathcal{A}^*}\mathbf{e}_k\|^2.
\end{aligned}
\end{equation}
Finally we can apply the spectral bound $\|A_{\mathcal{A}^*}\mathbf{e}_k\|^2 \ge \sigma_{\min}^2(A_{\mathcal{A}^*}) \|\mathbf{e}_k\|^2$ to yield the accelerated rate.
\end{proof}

\section{Computational Complexity of Proposed Cholesky Update Method} \subsection{Implementation Details of the Cholesky Update}
\label{app:cholesky_imple}

The key of our proposed Cholesky update framework is to utilize the properties of low-rank perturbations, allowing us to reuse updated Cholesky factors. This avoids the high computational cost of Singular Value Decomposition while efficiently maintaining both the projection operator and the null-space basis.

Let $\mathbf{G} = \mathbf{C}\mathbf{C}^\top$ be the original Gram matrix. The perturbed Gram matrix $\tilde{\mathbf{G}}$ is derived as:
\begin{equation}\begin{aligned}\tilde{\mathbf{G}} &= (\mathbf{C} + \mathbf{u}\mathbf{v}^\top)(\mathbf{C} + \mathbf{u}\mathbf{v}^\top)^\top \\
&= \mathbf{C}\mathbf{C}^\top + \mathbf{C}\mathbf{v}\mathbf{u}^\top + \mathbf{u}\mathbf{v}^\top\mathbf{C}^\top + \mathbf{u}(\mathbf{v}^\top\mathbf{v})\mathbf{u}^\top \\
&= \mathbf{G} + (\mathbf{C}\mathbf{v})\mathbf{u}^\top + \mathbf{u}(\mathbf{C}\mathbf{v})^\top + \|\mathbf{v}\|^2 \mathbf{u}\mathbf{u}^\top\end{aligned}
\end{equation}

Let $\mathbf{h} = \mathbf{C}\mathbf{v}$ and scalar $\lambda = \|\mathbf{v}\|^2$, we rewrite the expression as:
\begin{equation}\tilde{\mathbf{G}} = \mathbf{G} + \mathbf{h}\mathbf{u}^\top + \mathbf{u}\mathbf{h}^\top + \lambda \mathbf{u}\mathbf{u}^\top
\end{equation}

We define an auxiliary vector $\mathbf{m} = \mathbf{h} + \frac{\lambda}{2}\mathbf{u}$, then we have:
\begin{equation}
\mathbf{h}\mathbf{u}^\top + \mathbf{u}\mathbf{h}^\top + \lambda \mathbf{u}\mathbf{u}^\top = \mathbf{u}\mathbf{m}^\top + \mathbf{m}\mathbf{u}^\top = \frac{1}{2}(\mathbf{u}+\mathbf{m})(\mathbf{u}+\mathbf{m})^\top - \frac{1}{2}(\mathbf{u}-\mathbf{m})(\mathbf{u}-\mathbf{m})^\top
\end{equation}

Substituting this back into the Gram matrix expression yields: 
\begin{equation} \tilde{\mathbf{G}} = \mathbf{G} + \underbrace{\frac{1}{2}(\mathbf{u}+\mathbf{m})(\mathbf{u}+\mathbf{m})^\top}_{\text{Update}} - \underbrace{\frac{1}{2}(\mathbf{u}-\mathbf{m})(\mathbf{u}-\mathbf{m})^\top}_{\text{Downdate}} 
\end{equation}

Consequently, we do not need to explicitly reconstruct $\tilde{\mathbf{G}}$. Using standard \texttt{chol\_update} and \texttt{chol\_downdate} routines, we apply one rank-1 updates and one rank-1 downdate to the original factor $\mathbf{L}$, directly obtaining $\tilde{\mathbf{L}}$ such that $\tilde{\mathbf{L}}\tilde{\mathbf{L}}^\top = \tilde{\mathbf{G}}$.

With the updated factor $\tilde{\mathbf{L}}$, we define the constraint residual $\mathbf{r} = \tilde{\mathbf{C}}\mathbf{x}_0 - \mathbf{d}$. The projection formula is given by:
\begin{equation}
\mathbf{x}_{\text{eq}} = \mathbf{x}_0 - \tilde{\mathbf{C}}^\top ((\tilde{\mathbf{L}}^\top)^{-1} (\tilde{\mathbf{L}}^{-1} \mathbf{r}))
\end{equation}
Using $\tilde{\mathbf{L}}$, we solve this via two triangular substitutions:
\begin{enumerate}
\item \textbf{Forward Substitution:} $\tilde{\mathbf{L}} \mathbf{k} = \mathbf{r}$;
\item \textbf{Backward Substitution:} $\tilde{\mathbf{L}}^\top \mathbf{y} = \mathbf{k}$.
\end{enumerate}
The final projected solution is obtained by:
\begin{equation}
\tilde{\mathbf{x}}_\text{eq} = \mathbf{x}_0 - \tilde{\mathbf{C}}^\top \mathbf{y}
\end{equation}

Next, we construct the new null-space basis $\tilde{\mathbf{Z}}$ such that $\tilde{\mathbf{C}}\tilde{\mathbf{Z}} = \mathbf{0}$, employing a rotation-correction strategy. First, we compute the projection of the perturbation onto the original null space, $\mathbf{w} = \mathbf{Z}^\top \mathbf{v}$. To concentrate the perturbation, we construct a Householder reflection $\mathbf{H} = \mathbf{I} - 2\frac{\mathbf{q}\mathbf{q}^\top}{\|\mathbf{q}\|^2}$ with $\mathbf{q} = \mathbf{w} - \|\mathbf{w}\|\mathbf{e}_1$. Applying this to the original basis yields the transition basis \cite{householder1958unitary}:
\begin{equation}
\mathbf{Z}' = \mathbf{Z}\mathbf{H} = \left[ \mathbf{z}'_1, \mathbf{z}'_2, \dots, \mathbf{z}'_{n-m} \right]
\end{equation}

By properties of the Householder transformation, columns $2$ through $n-m$ of $\mathbf{Z}'$ satisfy $\mathbf{v}^\top \mathbf{z}'_j = 0$. Thus, for $j \ge 2$:
\begin{equation}
\tilde{\mathbf{C}} \mathbf{z}'_j = (\mathbf{C} + \mathbf{u}\mathbf{v}^\top)\mathbf{z}'_j = \mathbf{C}\mathbf{z}'_j + \mathbf{u}(\mathbf{v}^\top \mathbf{z}'_j) = \mathbf{0}
\end{equation}

This indicates that $\mathbf{Z}'_{:, 2:\text{end}}$ directly forms the majority of the new null space. The first column $\mathbf{z}'_1$ produces a non-zero residual $\tilde{\mathbf{C}}\mathbf{z}'_1 = \mathbf{u}(\mathbf{v}^\top \mathbf{z}'_1)$. To eliminate this, we seek a vector $\mathbf{p}$ satisfying $\tilde{\mathbf{C}}\mathbf{p} = \mathbf{u}$. Using the maintained $\tilde{\mathbf{L}}$, we solve $\tilde{\mathbf{C}}\tilde{\mathbf{C}}^\top \mathbf{y} = \mathbf{u}$ via triangular substitution and set $\mathbf{p} = \tilde{\mathbf{C}}^\top \mathbf{y}$.
Finally, adding the correction term to $\mathbf{z}'_1$ yields the complete new basis:
\begin{equation}
\tilde{\mathbf{Z}} = \left[ \mathbf{z}'_1 - \mathbf{v}^\top \mathbf{z}'_1 \mathbf{p}, \quad \mathbf{Z}'_ {:, 2:\text{end}} \right]
\end{equation}

\subsection{Complexity Analysis}
\label{app:choleskycomplexity}

Let $\mathbf{C}\in\mathbb{R}^{n_{\mathrm{eq}}\times n}$ denote the row full-rank equality-constraint matrix, and let $r=n-n_{\mathrm{eq}}$ be the dimension of its null space.  The SVD approach requires a complete factorization of the constraint matrix, incurring a time complexity of $\mathcal{O}(n^3)$. In contrast, the proposed framework exploits the low-rank nature of topological perturbations, effectively decoupling the high-dimensional problem into sequential vector-level operations. The computational cost of one complete iteration is decomposed as follows:
\begin{itemize}
    \item \textbf{Factor Maintenance:} We perform one rank-1 update and one rank-1 downdate on the Cholesky factor $\mathbf{L}$ of $\mathbf{G}=\mathbf{C}\mathbf{C}^\top \in \mathbb{R}^{n_{\mathrm{eq}}\times n_{\mathrm{eq}}}$. This step operates purely on the triangular factor, reducing the complexity to $\mathcal{O}(n_{\mathrm{eq}}^2)$, where $n_{\mathrm{eq}}$ is the number of equality constraints.
    \item \textbf{Equality Projection:} Computing the projected solution involves matrix-vector multiplications and triangular substitutions using the updated $\tilde{\mathbf{L}}$. The dominant cost is the matrix-vector product, scaling as $\mathcal{O}(n_{\mathrm{eq}} \cdot n)$.
    \item \textbf{Null Space Update:} The subspace rotation involves Householder transformations on the basis matrix $\mathbf{Z} \in \mathbb{R}^{n \times r}$, costing $\mathcal{O}(n \cdot r)$, and the correction requires solving for vector $\mathbf{p}$, adding a $\mathcal{O}(n_{\mathrm{eq}}\cdot n)$ cost regardless of the null-space dimension.
\end{itemize}

Since both $n_{\mathrm{eq}}$ and $r$ are bounded by $n$, the overall asymptotic complexity of our framework is $\mathcal{O}(n^2)$. This represents a theoretical speedup of order $n$ compared to standard SVD-based methods, transforming the expensive topological adaptation into a computationally lightweight process which is suitable for real-time applications.

\subsection{Detailed Analysis for Results on RGGs}
\label{app:rgg_analysis}

To validate the theoretical complexity analysis, we performed tests on Random Geometric Graphs of varying scales. Table \ref{tab:rgg_complexity_comparison} presents the overall performance comparison.

\begin{table}[htbp]
  \caption{Computation time comparison between SVD and Cholesky Update on random geometric graphs. The reported complexity values are empirical scaling exponents obtained by least-squares linear regression in log–log space.}
  \label{tab:rgg_complexity_comparison}
  \centering
  \begin{scriptsize}
  \begin{sc}
  \begin{tabular}{ccccc|ccccc}
    \toprule
    \multicolumn{5}{c|}{\textbf{Small-Scale ($N=100 \sim 1000$)}} & \multicolumn{5}{c}{\textbf{Large-Scale ($N=1500 \sim 5000$)}} \\
    \cmidrule(r){1-5} \cmidrule(l){6-10}
    $N$ & $K$ & SVD (ms) & Chol. (ms) & Speedup & $N$ & $K$ & SVD (ms) & Chol. (ms) & Speedup \\
    \midrule
    100  & 10 & 0.874     & 0.119   & 7.36$\times$   & 1500 & 38 & 307.441    & 8.187    & 37.55$\times$ \\
    200  & 14 & 2.906     & 0.290   & 10.02$\times$  & 2000 & 44 & 995.480    & 14.754   & 67.47$\times$ \\
    400  & 20 & 11.106    & 0.575   & 19.33$\times$  & 2500 & 50 & 2276.658   & 23.450   & 97.09$\times$ \\
    800  & 28 & 49.833    & 2.257   & 22.08$\times$  & 3000 & 54 & 4595.625   & 40.505   & 113.46$\times$ \\
    1000 & 31 & 83.666    & 3.583   & 23.35$\times$  & 5000 & 70 & 20596.582  & 122.586  & 168.02$\times$ \\
    \midrule
    \textbf{Complex.} & - & $\mathcal{O}(n^{1.99})$ & $\mathcal{O}(n^{1.46})$ & - &
    \textbf{Complex.} & - & $\mathcal{O}(n^{3.49})$ & $\mathcal{O}(n^{2.28})$ & - \\
    \bottomrule
  \end{tabular}
  \end{sc}
  \end{scriptsize}
\end{table}

Table \ref{tab:rgg_complexity_comparison} highlights the significant scalability gap between the SVD approach and our proposed Cholesky Update method. In the small-scale group from 100 to 1000, we observe empirical complexities of $\mathcal{O}(n^{1.99})$ for SVD and $\mathcal{O}(n^{1.46})$ for the Cholesky Update. These growth rates are much lower than their theoretical asymptotic bounds, which can be attributed to specific optimizations in underlying linear algebra libraries that maximize cache efficiency and instruction throughput for smaller matrices. However, even with these library-level enhancements benefiting the baseline, our method maintains superior performance, delivering substantial speedups ranging from 7.36$\times$ to 23.35$\times$.

As the number of nodes increases from 1500 to 5000, the SVD computation time exhibits a super-cubic growth trend with a fitted exponent of 3.49 in the large-scale group, reflecting the heavy burden of computation and memory data movement. In contrast, the proposed Cholesky Update maintains a near-quadratic growth rate of approximately 2.28, which closely matches the theoretical lower bound. A remaining limitation is that large-scale matrix operations can become memory-bound, as the increasing matrix size and data movement overhead can partially offset the arithmetic savings and push the observed scaling above the ideal quadratic rate. Nevertheless, the method still maintains a clear efficiency advantage over SVD, achieving over 168 times acceleration at 5000 nodes.

\subsection{Numerical Stability under Higher-Rank Contingencies}
\label{app:cholesky_stability}

Sequential Cholesky updates may accumulate numerical errors under finite precision, especially when multiple topology changes are applied consecutively. This issue becomes more visible for higher-rank contingencies and larger systems, where the update and downdate vectors can contain entries with very different magnitudes. To address this numerical concern, we use an adaptive rescaling strategy that preserves the exact low-rank perturbation while improving the conditioning of intermediate update vectors.

For a rank-one perturbation $\Delta \mathbf{C}=\mathbf{u}\mathbf{v}^\top$, the representation is not unique. For any scalar $\alpha>0$, we have
\begin{equation}
\Delta \mathbf{C}
=
\mathbf{u}\mathbf{v}^\top
=
(\alpha \mathbf{u})(\mathbf{v}/\alpha)^\top .
\end{equation}
In the Gram-matrix update derivation in Appendix~\ref{app:cholesky_imple}, we define
$\mathbf{h}=\mathbf{C}\mathbf{v}$,
$\lambda=\|\mathbf{v}\|_2^2$, and
$\mathbf{m}=\mathbf{h}+\frac{\lambda}{2}\mathbf{u}$.
After applying the above rescaling, the corresponding pair becomes $\alpha \mathbf{u}$ and $\mathbf{m}/\alpha$. Therefore, the exact perturbation remains unchanged, but the numerical magnitudes of the two vectors entering the Cholesky update and downdate can be balanced. In practice, we choose
\begin{equation}
\alpha=\sqrt{\frac{\|\mathbf{m}\|_2}{\|\mathbf{u}\|_2}},
\end{equation}
which equalizes the norms of $\alpha\mathbf{u}$ and $\mathbf{m}/\alpha$. This adaptive rescaling reduces extreme intermediate values, for example, on the IEEE 300-bus system, the largest intermediate magnitude is reduced from roughly $10^{10}$ to roughly $10^5$.

We evaluate both the original and adaptive Cholesky updates under fp32 and fp64 precision. The experiments cover all $N$-1 and $N$-2 contingencies, together with a subset of $N$-3 contingencies, on the IEEE 57-, 118-, and 300-bus systems. A scenario is counted as a failure if the sequential Cholesky update cannot be completed in the corresponding floating-point precision.

\begin{table}[h]
\centering
\scriptsize
\setlength{\tabcolsep}{3.5pt}
\caption{Failure counts of original and adaptive Cholesky updates under fp32 and fp64 precision. The percentage in parentheses denotes the failure rate among evaluated contingency scenarios.}
\label{tab:cholesky_stability}
\begin{sc}
\begin{tabular}{lcccccc}
\toprule
\multirow{2}{*}{System} &
\multirow{2}{*}{Contingency} &
\multirow{2}{*}{\# Scenarios} &
\multicolumn{2}{c}{FP32 Failures} &
\multicolumn{2}{c}{FP64 Failures} \\
\cmidrule(lr){4-5} \cmidrule(lr){6-7}
& & & Original & Adaptive & Original & Adaptive \\
\midrule
57-bus  & $N$-1 & 79     & 68 (86.1\%)     & 0 (0.0\%)   & 0 (0.0\%)   & 0 (0.0\%) \\
57-bus  & $N$-2 & 3024   & 2973 (98.3\%)   & 0 (0.0\%)   & 0 (0.0\%)   & 0 (0.0\%) \\
57-bus  & $N$-3 & 6048   & -              & 0 (0.0\%)   & -           & 0 (0.0\%) \\
\midrule
118-bus & $N$-1 & 177    & 174 (98.3\%)    & 0 (0.0\%)   & 0 (0.0\%)   & 0 (0.0\%) \\
118-bus & $N$-2 & 15502  & 15501 (99.9\%)  & 0 (0.0\%)   & 0 (0.0\%)   & 0 (0.0\%) \\
118-bus & $N$-3 & 31004  & -              & 0 (0.0\%)   & -           & 0 (0.0\%) \\
\midrule
300-bus & $N$-1 & 322    & 303 (94.1\%)    & 1 (0.3\%)   & 2 (0.6\%)   & 0 (0.0\%) \\
300-bus & $N$-2 & 51559  & 51405 (99.7\%)  & 279 (0.5\%) & 641 (1.2\%) & 0 (0.0\%) \\
300-bus & $N$-3 & 103118 & -              & 558 (0.5\%) & -           & 0 (0.0\%) \\
\bottomrule
\end{tabular}
\end{sc}
\end{table}

As shown in Table~\ref{tab:cholesky_stability}, the original sequential update is highly unstable in fp32 and becomes almost unusable for $N$-1 and $N$-2 contingencies. Although fp64 substantially improves stability, the original update still produces a small number of failures on the IEEE 300-bus system. In contrast, the adaptive rescaling removes all observed fp64 failures across all evaluated systems and contingency levels. Under fp32, it also eliminates all failures on the 57- and 118-bus systems and leaves only a small failure rate on the larger 300-bus system. These results indicate that the numerical stability issue of sequential Cholesky updates can be effectively controlled by adaptive rescaling, while the exact low-rank perturbation and the theoretical update structure remain unchanged.

\section{Minimum-Cost Gas Transport Problem}
\label{appendix:min_cost}

We further evaluate the proposed method on a minimum-cost gas transport problem using the GasLib-135 benchmark \cite{gaslib}. Given a gas network $G=(V,E)$, each node $v\in V$ represents a supply or demand point, and each edge $e\in E$ is associated with capacity limits and transportation costs. The goal is to find a minimum-cost flow allocation while satisfying nodal flow-balance constraints and edge-capacity constraints. Each test instance corresponds to an operating scenario with different demand levels and possible topology changes, such as edge removals or failure cases. We evaluate both seen and unseen scenarios to assess generalization across network conditions. The optimization problem is formulated as follows:

\begin{equation}
\label{eq:static_flow_exp}
\begin{aligned}
\min_{\mathbf{g}^{(k)},\,\mathbf{f}^{(k)}} \quad
& \sum_{s \in \mathcal{S}} \alpha_s g_s^{(k)}
  + \sum_{\ell \in \mathcal{L}^{(k)}} \rho_\ell \left| f_\ell^{(k)} \right| \\
\text{s.t.}\quad
& \sum_{s \in \mathcal{S}(v)} g_s^{(k)}
  + \sum_{\ell \in \mathcal{L}^{(k)}} A_{v\ell}^{(k)} f_\ell^{(k)}
  = d_v,
  \quad \forall v \in \mathcal{V}\setminus\{v_{\mathrm{ref}}\}, \\
& \sum_{s \in \mathcal{S}} g_s^{(k)}
  = \sum_{r \in \mathcal{R}} d_r, \\
& 0 \leq g_s^{(k)} \leq \overline{g}_s,
  \quad \forall s \in \mathcal{S}, \\
& -\overline{f}_\ell \leq f_\ell^{(k)} \leq \overline{f}_\ell,
  \quad \forall \ell \in \mathcal{L}^{(k)}, \\
& k \in \{0\} \cup \mathcal{K}.
\end{aligned}
\end{equation}
where 
\[
\mathcal{L}^{(0)} = \mathcal{L}, 
\qquad
\mathcal{L}^{(k)} = \mathcal{L} \setminus \{k\}, \quad k \in \mathcal{K}.
\]

\begin{table}[h]
\centering
\scriptsize
\caption{HGNN-only prediction quality on the minimum-cost gas transport problem. Metrics except forward time are formatted as ``Seen / Unseen''.}
\label{tab:gas_hgnn_quality}
\begin{sc}
\begin{tabular}{lcccc}
\toprule
Method & Forward Time (ms) & Max Eq. Viol. & Max Ineq. Viol. & Opt. Gap (\%) \\
\midrule
HGNN & $3.810 \pm 2.074$ & $473.6 / 2337.0$ & $100.5 / 416.8$ & $0.008 / 0.009$ \\
\bottomrule
\end{tabular}
\end{sc}
\end{table}

\begin{table}[h]
\centering
\scriptsize
\caption{SKM-layer performance on the minimum-cost gas transport problem. The optimality gap is reported as mean value in the format ``Seen / Unseen''.}
\label{tab:gas_skm_performance}
\begin{sc}
\begin{tabular}{lcccccccc}
\toprule
Method & Eq. Proj. & Eq. Spd. & Ineq. Iter. & Ineq. Time & Ineq. Spd. & SKM Time & SKM Spd. & Opt. Gap (\%) \\
\midrule
T-SKM  & $2.19 \pm 1.06$ & $1.00\times$ & $15.21\,(128)$ & $0.96 \pm 1.22$ & $1.00\times$ & $3.15$ & $1.00\times$ & $0.006 / 0.019$ \\
AT-SKM & $0.53 \pm 0.35$ & $4.13\times$ & $5.33\,(78)$   & $0.37 \pm 0.71$ & $2.59\times$ & $0.90$ & $3.50\times$ & $0.006 / 0.019$ \\
\bottomrule
\end{tabular}
\end{sc}
\end{table}

Table~\ref{tab:gas_hgnn_quality} shows that the HGNN-only predictor can obtain small mean optimality gaps, but it suffers from substantial worst-case equality and inequality constraint violations, especially on unseen scenarios. This indicates that direct neural prediction alone is insufficient for safety-critical constrained optimization.

In contrast, Table~\ref{tab:gas_skm_performance} shows that both T-SKM-Net and AT-SKM-Net achieve zero equality and inequality constraint violations on both seen and unseen scenarios. Compared with T-SKM-Net, AT-SKM-Net reduces the equality projection time from $2.19$ ms to $0.53$ ms, yielding a $4.13\times$ speedup, and reduces the inequality iteration time from $0.96$ ms to $0.37$ ms, yielding a $2.59\times$ speedup. Overall, AT-SKM-Net reduces the total SKM computation time from $3.15$ ms to $0.90$ ms, corresponding to a $3.50\times$ acceleration while preserving strict feasibility and maintaining comparable mean optimality gaps.

\section{Experimental Settings and Additional Results for N-1 Security-Constrained DC-OPF}
\label{appendix:experiments}

\subsection{Mathematical Formulation of N-1 Security-Constrained DC-OPF}
\label{appendix:scopf}

The N-1 Security-Constrained DC Optimal Power Flow problem aims to determine the optimal generation dispatch that minimizes total operating costs while ensuring system feasibility under both normal conditions and any single transmission line outage. Let $\mathcal{K}$ denote the set of all considered scenarios, where $k=0$ represents the base case and each remaining $k \in \mathcal{K}$ corresponds to a specific line failure contingency. Under a given scenario $k$, the network topology changes accordingly; we define $\mathcal{L}^{(k)}$ as the set of remaining active lines and $\mathcal{N}^{(k)}(i)$ as the set of neighboring buses for bus $i$ in that topology. The optimization problem is formulated as follows:

\begin{equation}
\label{eq:scopf_full}
\begin{aligned}
\min_{\mathbf{P}_\mathcal{G}, \boldsymbol{\delta}} \quad & \sum_{i \in \mathcal{G}} \left(\frac{1}{2} c_i P_{G,i}^2 + b_i P_{G,i}\right) \\
\text{s.t.}\quad & P_{G,i} - P_{D,i} = \sum_{j \in \mathcal{N}^{(k)}(i)} B_{ij} (\delta_i^{(k)} - \delta_j^{(k)}), \quad \forall i \in \mathcal{B}, \forall k \in \{0\} \cup \mathcal{K} \\
& P_{G,i}^{\min} \leq P_{G,i} \leq P_{G,i}^{\max}, \quad \forall i \in \mathcal{G} \\
& -P_{ij}^{\max} \leq B_{ij}(\delta_i^{(k)} - \delta_j^{(k)}) \leq P_{ij}^{\max}, \quad \forall (i, j) \in \mathcal{L}^{(k)}, \forall k \in \{0\} \cup \mathcal{K}
\end{aligned}
\end{equation}
where the sets $\mathcal{B}$, $\mathcal{G} \subseteq \mathcal{B}$, and $\mathcal{L}$ represent buses, generators, and transmission lines, respectively. The decision variables consist of the active power outputs $\mathbf{P}_\mathcal{G} = \{P_{G,i}\}_{i \in \mathcal{G}}$ and the voltage phase angles $\boldsymbol{\delta}^{(k)} = \{\delta_i^{(k)}\}_{i \in \mathcal{B}}$ for each scenario. It is important to note that the generation dispatch $P_G$ is preventive and shared across all scenarios, whereas the system state $\delta^{(k)}$ adapts to each specific topology. In the objective function, $c_i$ and $b_i$ denote the quadratic and linear cost coefficients for generator $i$. The parameter $P_{D,i}$ represents the active power demand at bus $i$, and $B_{ij}$ is the line susceptance defined as a positive value $1/X_{ij}$. For non-generator buses $i \in \mathcal{B} \setminus \mathcal{G}$, we set $P_{G,i} = 0$. The terms $P_{G,i}^{\min/\max}$ and $P_{ij}^{\max}$ define the generator capacity limits and transmission line thermal limits, respectively.

\subsection{N-1 Security-Constrainted DC-OPF Dataset Generation}

To validate the effectiveness of the proposed method, we employ the IEEE 57, 118, and 300-bus systems as experimental benchmarks. In constructing the N-1 contingency dataset, we first screen all potential single-line outage scenarios, retaining only those that are physically feasible. Here, a feasible scenario is defined as one where the post-contingency network topology remains connected without forming islands, and the OPF solver successfully converges to a feasible solution. \cref{tab:dataset_statistics} summarizes the statistics for each test system, including the number of buses, total branches, and the count of valid N-1 contingencies with the original scenario after screening.

For each selected topological scenario, we generate 100 independent operating instances by introducing random perturbations to the nodal loads. Specifically, the load vectors are obtained by uniformly scaling the baseline load profile:
\begin{equation}P_D^{(i)} = \alpha_i \cdot P_D^{\text{base}}, \quad \alpha_i \sim \mathcal{U}(0.8, 1.2),
\end{equation}
where $P_D^{\text{base}}$ represents the baseline load values, and $\alpha_i$ is a random scaling factor drawn from a uniform distribution over the interval $[0.8, 1.2]$. This generation procedure ensures that the dataset encompasses a wide range of load fluctuations.

Regarding dataset partitioning, to evaluate the model's generalization capability to unseen contingencies, we randomly divide all valid N-1 scenarios into two subsets: 70\% are designated as \textit{seen} scenarios for the training phase, while the remaining 30\% serve as \textit{unseen} scenarios. The training set consists exclusively of samples generated from the seen scenarios, whereas the test set comprehensively covers instances from both seen and unseen topologies.

\begin{table}[h]
\centering
\scriptsize
\caption{Settings of N-1 contingency Dataset and Hyperparameters of SKM Sampling}
\label{tab:dataset_statistics}
\begin{sc}
\begin{tabularx}{0.95\textwidth}{lXXXXXl}
\toprule
Test System & Buses & Branches & Removable & Scenarios & $\beta$ & $\rho$\\
\midrule
IEEE 57-Bus  & 57  & 80  & 79  & 80  & 10 & 0.8 \\
IEEE 118-Bus & 118 & 186 & 177 & 178 & 50 & 0.8 \\
IEEE 300-Bus & 300 & 411 & 322 & 323 & 60 & 0.8 \\
\bottomrule
\end{tabularx}
\end{sc}
\end{table}

\subsection{Experimental Hyperparameters and Loss Function Configuration}

In terms of model architecture and training configuration, we employ a Heterogeneous Graph Neural Network (HGNN) consisting of two message-passing layers, with the hidden dimension uniformly set to 128. We utilize the Adam optimizer for training, initialized with a learning rate of $5e{-4}$. A StepLR scheduler is applied to implement stepwise decay:
\begin{equation}
    \eta_t = \eta_0 \cdot \gamma^{\lfloor t / S \rfloor},
\end{equation}
where $\eta_0$ denotes the initial learning rate, $S = 25$ represents the step size, and $\gamma = 0.8$ is the decay factor. The batch size is set to 64 during the training phase. Conversely, the testing phase utilizes a batch size of 1 to simulate real-world scenarios requiring real-time inference for individual samples.

The composite loss function comprises two components: regression loss and classification loss. Following the methodology of T-SKM \cite{t-skm}, the regression loss adopts the Mean Squared Error (MSE) formulation, calculated as a weighted sum of the pre-projection and post-projection losses:
\begin{equation}
    L_{reg} = 0.9 \cdot \text{MSE}(\hat{\mathbf{y}}_{pre}, \mathbf{y}) + 0.1 \cdot \text{MSE}(\hat{\mathbf{y}}_{proj}, \mathbf{y}),
\end{equation}
where $\hat{\mathbf{y}}_{pre}$ represents the raw model output and $\hat{\mathbf{y}}_{proj}$ denotes the output after constraint projection. This combination enables the model to learn the primary feature mapping while simultaneously encouraging satisfaction of physical constraints.

For the active-set prediction task, we employ the Binary Focal Loss to address the issue of positive-negative sample imbalance \cite{lin2017focal}:
\begin{equation}
    L_\text{focal}(p_t) = -\alpha_t (1 - p_t)^{\gamma} \log(p_t),
\end{equation}
where $p_t$ is the model's estimated probability for the ground-truth class. The focusing parameter $\gamma = 2.0$ down-weights easy-to-classify examples, while the balancing parameter $\alpha = 0.5$ adjusts the relative importance of positive and negative samples. The total loss function is expressed as:
\begin{equation}
    L_\text{total} = L_\text{reg} + \lambda(t) \cdot L_\text{cls}.
\end{equation}
Here, $\lambda(t)$ follows a linear warmup schedule, increasing linearly from 0 to 1 over the first 40 training epochs. This strategy allows the model to prioritize the regression task during the initial training phase before gradually integrating the classification objective.

Active-set labels are generated using a relative slack criterion. For a two-sided inequality $l_i \le g_i(x) \le u_i$, the constraint is labeled active if its slack to either bound is less than $1\%$ of the feasible interval $u_i-l_i$. This threshold is used only for training the active-set predictor. In each iteration, the number of sampled constraints corresponds to approximately 10\% of the total constraint set. The specific configurations for the sampling batch size $\beta$ and the hybrid sampling ratio $\rho$ across different test systems are detailed in Table \ref{tab:dataset_statistics}.

\subsection{Hyperparameter Sensitivity of Hybrid Sampling Strategy}
\label{appendix:hybrid_sampling_sensitivity}
To analyze the sensitivity of our proposed active-set predict method to the batch size parameter, we conducted a controlled experiment using ground truth data. In this setup, we assume the availability of a perfect predictor by directly utilizing the actual active sets derived from the optimal solutions to guide the sampling process. We compared the convergence performance of the standard uniform sampling against the ground truth guided hybrid sampling across varying sampling ratios on IEEE 118 bus systems.

\begin{figure}[htbp]
    \centering
    \begin{subfigure}[b]{0.48\textwidth}
        \centering
        \includegraphics[width=\linewidth]{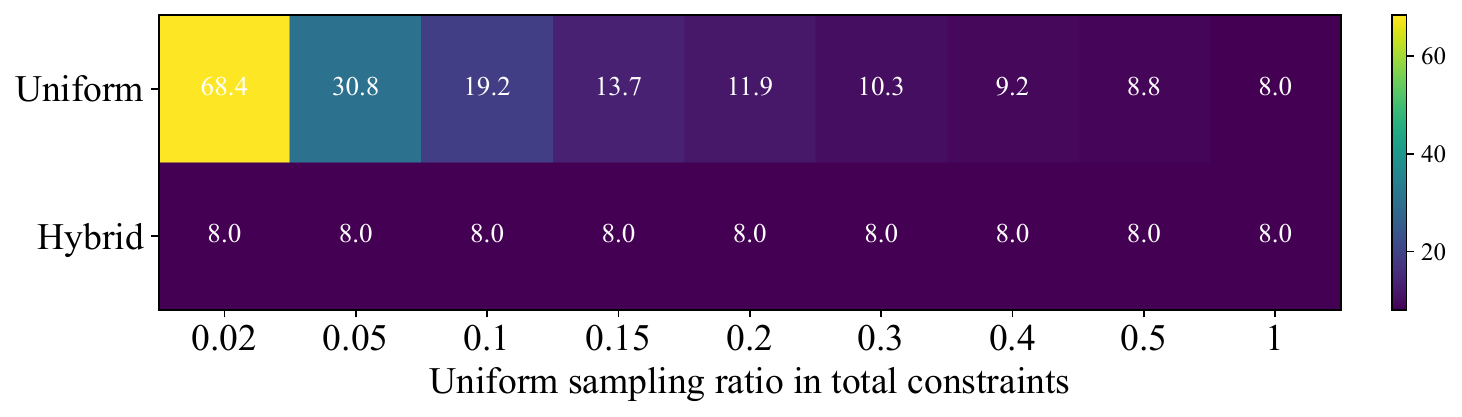}
        \caption{IEEE 118 Mean Iterations}
        \label{fig:mean_118}
    \end{subfigure}
    \hfill
    \begin{subfigure}[b]{0.48\textwidth}
        \centering
        \includegraphics[width=\linewidth]{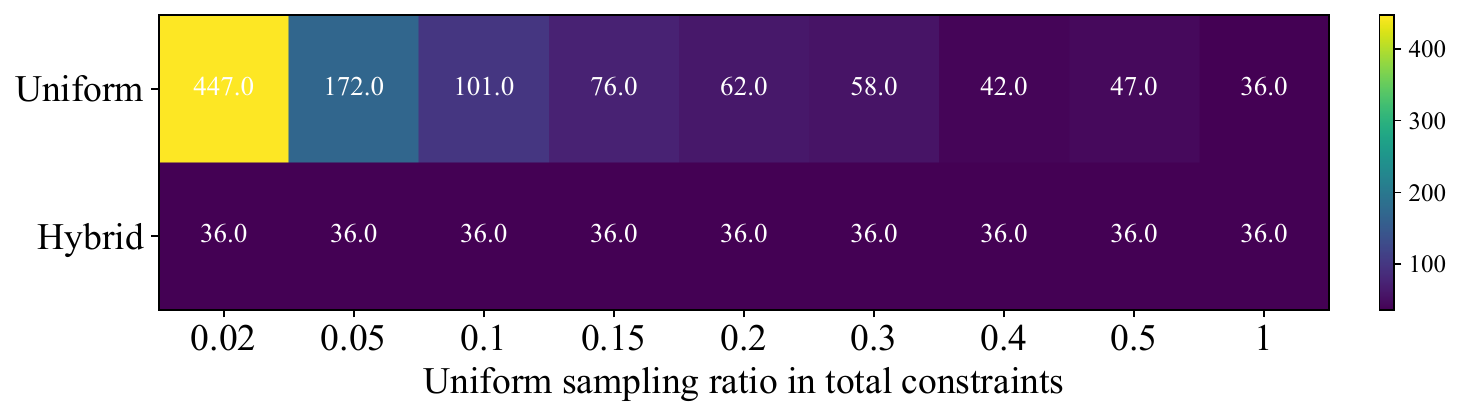}
        \caption{IEEE 118 Max Iterations}
        \label{fig:max_118}
    \end{subfigure}
    
    \begin{subfigure}[b]{0.48\textwidth}
        \centering
        \includegraphics[width=\linewidth]{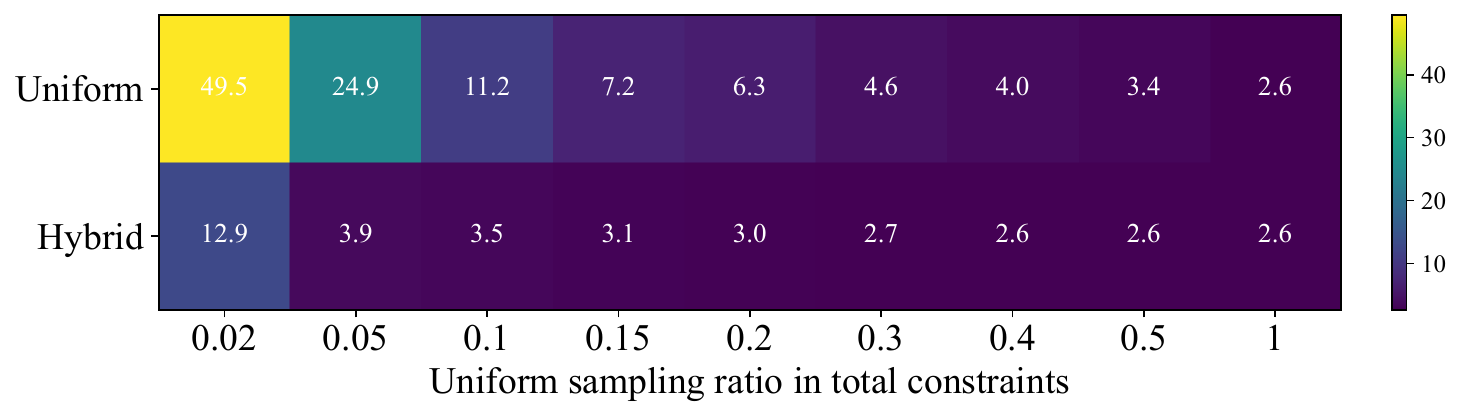}
        \caption{IEEE 300 Mean Iterations}
        \label{fig:mean_300}
    \end{subfigure}
    \hfill
    \begin{subfigure}[b]{0.48\textwidth}
        \centering
        \includegraphics[width=\linewidth]{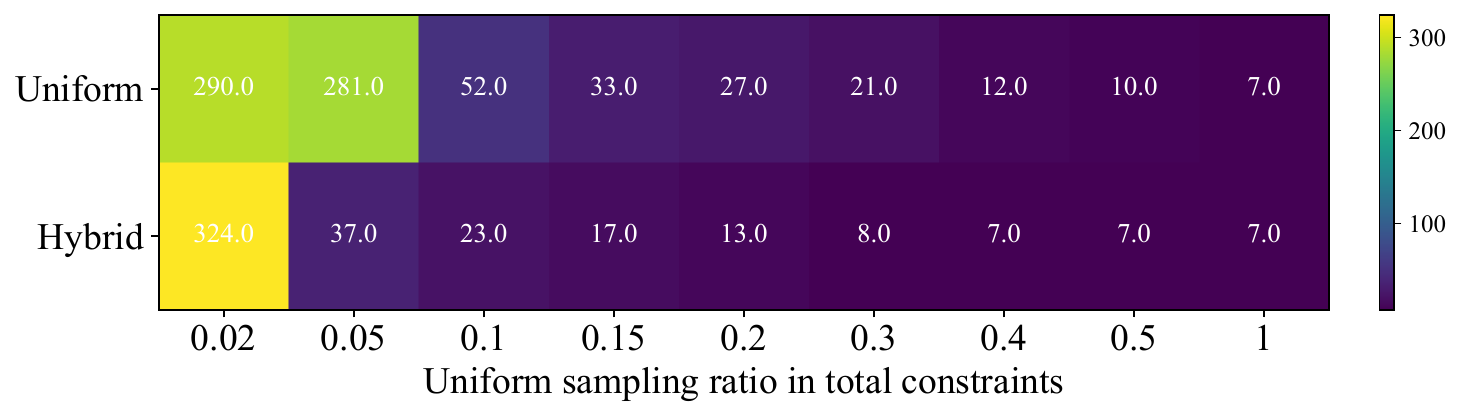}
        \caption{IEEE 300 Max Iterations}
        \label{fig:max_300}
    \end{subfigure}
    
    \caption{Sensitivity analysis of sampling strategies with a perfect predictor on IEEE 118-bus and 300-bus systems. Comparison between uniform sampling and ground truth guided hybrid sampling across varying sampling ratios.}
    \label{fig:sensitivity_analysis}
\end{figure}

Figure~\ref{fig:sensitivity_analysis} presents the sensitivity analysis where the SKM iterative process is initialized using the solution predicted by the HGNN. 

For the IEEE 118 bus system shown in Figure~\ref{fig:mean_118} and Figure~\ref{fig:max_118}, the results demonstrate that the hybrid sampling strategy exhibits remarkable stability when the active set is accurately identified. Regardless of the proportion of additional random sampling, both the mean and maximum iteration counts remain at a consistently low level. This implies that once the true active constraints are covered, the algorithm is effectively decoupled from the interference of massive inactive constraints. 

In the larger IEEE 300 bus system shown in Figure~\ref{fig:mean_300} and Figure~\ref{fig:max_300}, the hybrid strategy shows a significantly sharper improvement in convergence speed compared to uniform sampling. With only a minimal amount of random sampling included, the hybrid approach allows the iteration count to drop rapidly to its optimal level. Specifically, Figure 3d reveals that the hybrid strategy effectively suppresses worst case scenarios, achieving performance with about a tenth of the sampling effort that would otherwise require full set scanning in the uniform approach.

\subsection{Detailed Time Breakdown for AT-SKM-Net on SC DC-OPF Problems}
\label{appendix:time_breakdown_dcopf}

Using the same models and experimental settings as in the main SC DC-OPF experiments, we provide a detailed runtime breakdown of T-SKM-Net and AT-SKM-Net. Specifically, we separately measure the time spent in the equality projection stage and the inequality SKM iteration stage, which correspond to the two main computational components accelerated by the proposed Cholesky update and active-set-guided sampling strategy, respectively.

\begin{table}[h]
\centering
\scriptsize
\caption{Detailed runtime breakdown of T-SKM-Net and AT-SKM-Net on SC DC-OPF problems. AT-SKM-AC denotes the full AT-SKM-Net with both active-set-guided sampling and Cholesky update. Runtime values are reported in milliseconds.}
\label{tab:time_breakdown_dcopf}
\begin{tabular}{llccccc}
\toprule
\textsc{System} & \textsc{Method} & \textsc{Eq. Proj. Time}
& \textsc{Eq. Proj. Spd.} & \textsc{Ineq. Iter. Time}
& \textsc{Ineq. Iter. Spd.} & \textsc{Total SKM Spd.} \\
\midrule
57-bus  & T-SKM-Net  & $0.56 \pm 0.34$ & $1.00\times$  & $0.41 \pm 1.20$ & $1.00\times$ & $1.00\times$ \\
57-bus  & AT-SKM-AC & $0.18 \pm 0.02$ & $3.11\times$  & $0.06 \pm 0.07$ & $6.83\times$ & $4.04\times$ \\
\midrule
118-bus & T-SKM-Net  & $1.56 \pm 0.06$ & $1.00\times$  & $2.13 \pm 2.51$ & $1.00\times$ & $1.00\times$ \\
118-bus & AT-SKM-AC & $0.30 \pm 0.03$ & $5.20\times$  & $1.01 \pm 1.38$ & $2.11\times$ & $2.81\times$ \\
\midrule
300-bus & T-SKM-Net  & $7.65 \pm 1.54$ & $1.00\times$  & $2.07 \pm 4.59$ & $1.00\times$ & $1.00\times$ \\
300-bus & AT-SKM-AC & $0.70 \pm 0.02$ & $10.92\times$ & $0.32 \pm 0.20$ & $6.46\times$ & $9.53\times$ \\
\bottomrule
\end{tabular}
\end{table}

As shown in Table~\ref{tab:time_breakdown_dcopf}, AT-SKM-Net substantially reduces the equality projection time compared with T-SKM-Net. Moreover, the acceleration becomes more pronounced as the network size increases: the equality projection speedup grows from $3.11\times$ on the 57-bus system to $10.92\times$ on the 300-bus system. This trend is consistent with the theoretical motivation of the Cholesky update, whose advantage becomes more significant when the equality projection involves larger constraint matrices.

The inequality SKM iteration stage is also accelerated. Compared with T-SKM-Net, AT-SKM-Net reduces the inequality iteration time by $6.83\times$, $2.11\times$, and $6.46\times$ on the 57-, 118-, and 300-bus systems, respectively. These reductions reflect the benefit of active-set-guided sampling, which avoids spending excessive iterations on inactive or redundant constraints. Overall, AT-SKM-Net achieves total SKM speedups of $4.04\times$, $2.81\times$, and $9.53\times$ across the three systems, which is consistent with the main-paper conclusion that the proposed framework accelerates both equality projection and inequality iteration while preserving strict feasibility.

\subsection{Robustness under High-Load Operating Conditions}
\label{appendix:high_load_robustness}

To further evaluate the robustness of proposed AT-SKM-Net, we construct an additional high-load test regime for the SC DC-OPF problem. In the main experiments, load perturbations are sampled from the range $[0.8, 1.2]$ around the nominal demand profile. Here, we shift the load range to $[1.0, 1.2]$, thereby concentrating the test instances in more heavily loaded operating conditions. Such cases are more challenging because line-flow and generation-capacity constraints are more likely to approach their limits, making the active-set pattern more sensitive to both load variations and topology changes.

\begin{table}[h]
\centering
\scriptsize
\setlength{\tabcolsep}{2.2pt}
\caption{Performance of T-SKM-Net and AT-SKM-Net under high-load SC DC-OPF conditions. Load perturbations are sampled from $[1.0, 1.2]$. Runtime values are reported in milliseconds. The inequality iteration column reports the mean number of iterations, with the maximum shown in parentheses.}
\label{tab:high_load_robustness}
\begin{sc}
\begin{tabular}{llccccccc}
\toprule
System & Method &
Eq. Proj. Time &
Eq. Spd. &
Ineq. Iter. &
Ineq. Time &
Ineq. Spd. &
SKM Spd. &
Mean Gap (\%) \\
\midrule
57-bus  & T-SKM-Net  & $0.42 \pm 0.14$ & $1.00\times$ & $16.00\,(133)$ & $0.70 \pm 1.22$ & $1.00\times$ & $1.00\times$ & $0.014$ \\
57-bus  & AT-SKM-AC & $0.18 \pm 0.04$ & $2.33\times$ & $2.75\,(12)$   & $0.13 \pm 0.14$ & $5.38\times$ & $3.61\times$ & $0.010$ \\
\midrule
118-bus & T-SKM-Net  & $1.81 \pm 0.59$ & $1.00\times$ & $18.59\,(160)$ & $1.45 \pm 2.80$ & $1.00\times$ & $1.00\times$ & $0.004$ \\
118-bus & AT-SKM-AC & $0.37 \pm 0.22$ & $4.89\times$ & $9.21\,(56)$   & $0.82 \pm 1.56$ & $1.77\times$ & $2.74\times$ & $0.004$ \\
\midrule
300-bus & T-SKM-Net  & $7.05 \pm 0.51$ & $1.00\times$ & $25.49\,(248)$ & $1.78 \pm 3.02$ & $1.00\times$ & $1.00\times$ & $0.136$ \\
300-bus & AT-SKM-AC & $0.71 \pm 0.02$ & $9.93\times$ & $5.34\,(28)$   & $0.42 \pm 0.58$ & $4.24\times$ & $7.81\times$ & $0.136$ \\
\bottomrule
\end{tabular}
\end{sc}
\end{table}

As shown in Table~\ref{tab:high_load_robustness}, AT-SKM-Net continues to achieve clear acceleration over T-SKM-Net under the high-load regime. The equality projection stage is consistently accelerated, with the speedup increasing from $2.33\times$ on the 57-bus system to $9.93\times$ on the 300-bus system. This again confirms that the Cholesky update becomes increasingly beneficial as the system size grows.

The inequality iteration stage also remains robust under heavier loading. Compared with T-SKM-Net, AT-SKM-Net reduces the mean number of inequality iterations from $16.00$ to $2.75$ on the 57-bus system, from $18.59$ to $9.21$ on the 118-bus system, and from $25.49$ to $5.34$ on the 300-bus system. The maximum iteration count is also substantially reduced across all systems, indicating that active-set-guided sampling mitigates difficult tail cases even when more constraints are close to becoming active.

These results suggest that the HGNN predictor remains effective in more stressed operating conditions. The warm-start prediction continues to provide useful initialization, while the active-set predictor still concentrates sampling on the most relevant constraints. As a result, AT-SKM-Net achieves total SKM speedups of $3.61\times$, $2.74\times$, and $7.81\times$ on the 57-, 118-, and 300-bus systems, respectively, while maintaining comparable mean optimality gaps to T-SKM-Net.

\subsection{Active-Set Prediction Quality}
\label{appendix:active_set_quality}

\begin{table}[h]
\footnotesize
  \caption{Active set prediction performance on test ``seen / unseen'' scenarios.}
  \label{tab:active_set_metrics}
  \centering
  \begin{scriptsize}
  \begin{sc}
  \begin{tabular}{lccc}
    \toprule
    System & Accuracy & Precision & Recall \\
    \midrule
    IEEE 57-Bus  & 0.9995 / 0.9995 & 0.9814 / 0.9802 & 0.9804 / 0.9815 \\
    IEEE 118-Bus & 0.9995 / 0.9995 & 0.9988 / 0.9989 & 0.9948 / 0.9940 \\
    IEEE 300-Bus & 0.9991 / 0.9991 & 0.9647 / 0.9649 & 0.9834 / 0.9834 \\
    \bottomrule
  \end{tabular}
  \end{sc}
  \end{scriptsize}
\end{table}

To validate that the theoretical acceleration potential established in the sensitivity analysis is achievable in practice, we evaluate the actual classification performance of the proposed HGNN predictor.

Table~\ref{tab:active_set_metrics} details the Accuracy, Precision, and Recall metrics across three test systems. We explicitly distinguish between ``seen'' and ``unseen'' scenarios to assess the model's generalization capability.

Given that active constraints constitute only a tiny fraction of the total set yet strictly define the optimal solution, missing even a single binding constraint can severely stall convergence. Consequently, the Recall rate is as the most paramount metric. Table~\ref{tab:active_set_metrics} shows that the HGNN consistently achieves a Recall exceeding 98\% across all systems. This high coverage guarantees that the predictor captures nearly all binding constraints required to define the active manifold, allowing SKM to focus on the feasible boundary and minimizing the reliance on random exploration to recover missed information. These recall values provide empirical support for the coverage condition in Assumption~\ref{ass:active_covering}, namely $\mathcal{A}^* \subseteq \hat{\mathcal{A}}$, up to the small prediction error observed in finite test samples.

In addition to reliability, the method maintains consistently high precision from 96.5\% to 99.9\%, ensuring that the predicted candidate sets are compact with minimal redundancy. Such compactness is essential for efficiency as it prevents the dilution of sampling probability. Moreover, the framework demonstrates that HGNN remains robustness to topological changes. The negligible performance gap between seen and unseen scenarios indicates that the HGNN has captured the underlying physical laws governing constraint activation, rather than overfitting to specific patterns. This robustness ensures that AT-SKM sustains its acceleration performance even under unexpected line failures.

\subsection{Robustness to Active-Set Perturbations}
\label{appendix:active_set_perturbation}

We further analyze the robustness of the proposed active-set-guided sampling strategy when the predicted active set is imperfect. In practice, the prediction errors of an active-set predictor can be divided into two types: false negatives, where truly active constraints are dropped from the predicted pool, and false positives, where inactive constraints are added to the predicted pool. Both types of errors weaken the sampling guidance, but their effects are not symmetric.

To isolate these two effects, we conduct a two-dimensional perturbation experiment using the oracle active set. Starting from the true active set, we randomly drop a prescribed fraction of true active constraints and add a prescribed number of inactive constraints. Both the add and drop ratios are normalized by the size of the true active set. We then run the SKM layer with the perturbed active-set pool and report the mean number of inequality iterations.

\begin{table}[h]
\centering
\footnotesize
\setlength{\tabcolsep}{4pt}
\caption{Mean number of SKM inequality iterations under active-set perturbations on the IEEE 118-bus system. Rows denote the ratio of added inactive constraints, and columns denote the ratio of dropped true active constraints. Both ratios are normalized by the true active-set size.}
\label{tab:active_set_perturb_118}
\begin{sc}
\begin{tabular}{lccccccc}
\toprule
\multirow{2}{*}{Add Ratio} & \multicolumn{7}{c}{Drop Ratio} \\
\cmidrule(lr){2-8}
 & 0.00 & 0.05 & 0.10 & 0.15 & 0.20 & 0.25 & 0.50 \\
\midrule
0.0 & 15.68 & 18.04 & 20.02 & 22.71 & 24.41 & 24.71 & 29.61 \\
0.5 & 15.85 & 20.64 & 29.60 & 33.44 & 39.12 & 39.71 & 42.53 \\
1.0 & 16.37 & 21.41 & 29.04 & 35.27 & 37.91 & 42.79 & 64.02 \\
1.5 & 17.04 & 21.34 & 27.89 & 35.81 & 39.81 & 41.90 & 63.72 \\
2.0 & 17.64 & 22.10 & 30.32 & 36.28 & 41.06 & 42.42 & 63.08 \\
3.0 & 19.12 & 24.32 & 31.16 & 36.93 & 41.80 & 44.11 & 65.27 \\
\bottomrule
\end{tabular}
\end{sc}
\end{table}

\begin{table}[h]
\centering
\footnotesize
\setlength{\tabcolsep}{4pt}
\caption{Mean number of SKM inequality iterations under active-set perturbations on the IEEE 300-bus system. Rows denote the ratio of added inactive constraints, and columns denote the ratio of dropped true active constraints.}
\label{tab:active_set_perturb_300}
\begin{sc}
\begin{tabular}{lccccccc}
\toprule
\multirow{2}{*}{Add Ratio} & \multicolumn{7}{c}{Drop Ratio} \\
\cmidrule(lr){2-8}
 & 0.00 & 0.05 & 0.10 & 0.15 & 0.20 & 0.25 & 0.50 \\
\midrule
0.0 & 3.69 & 4.15 & 7.68 & 8.88 & 10.99 & 13.96 & 19.24 \\
0.5 & 3.69 & 4.51 & 9.65 & 10.39 & 14.43 & 16.55 & 23.97 \\
1.0 & 3.72 & 4.78 & 11.53 & 12.50 & 17.17 & 23.93 & 33.14 \\
1.5 & 3.77 & 5.25 & 12.03 & 13.27 & 20.15 & 24.80 & 39.51 \\
2.0 & 3.80 & 4.79 & 13.36 & 15.24 & 23.41 & 28.90 & 43.07 \\
3.0 & 4.04 & 5.79 & 17.90 & 21.70 & 34.99 & 40.32 & 69.75 \\
\bottomrule
\end{tabular}
\end{sc}
\end{table}

Tables~\ref{tab:active_set_perturb_118} and~\ref{tab:active_set_perturb_300} show that both false positives and false negatives degrade the efficiency of active-set-guided sampling. However, dropping true active constraints is generally more harmful than adding inactive constraints. On the IEEE 118-bus system, when no inactive constraints are added, increasing the drop ratio from $0$ to $0.5$ increases the mean iteration count from $15.68$ to $29.61$. In contrast, when no true active constraints are dropped, increasing the add ratio from $0$ to $3.0$ only increases the mean iteration count from $15.68$ to $19.12$. A similar trend appears on the IEEE 300-bus system, where the iteration count increases from $3.69$ to $19.24$ as the drop ratio grows from $0$ to $0.5$, but only from $3.69$ to $4.04$ as the add ratio grows from $0$ to $3.0$ without dropping true active constraints. For reference, the original T-SKM-Net requires $34.60$ inequality iterations on the IEEE 118-bus system and $24.72$ iterations on the IEEE 300-bus system.

This behavior is consistent with the SKM update rule. In each iteration, SKM samples a subset of constraints and projects onto the most violated constraint within that sampled subset. Therefore, inactive false-positive constraints do not necessarily cause harmful projections; their main effect is to dilute the probability of sampling truly active constraints. In contrast, false negatives remove truly active constraints from the guided sampling pool, forcing them to be recovered only through the uniform sampling branch.

Let $m$ denote the total number of constraints, $s$ the size of the predicted active-set pool, and $\beta_g$ and $\beta_u$ the guided and uniform sample sizes, respectively. For a true active constraint that remains in the predicted pool, its probability of being sampled is
\begin{equation}
p_{\mathrm{keep}}
=
\frac{\beta_g}{s}
+
\left(1-\frac{\beta_g}{s}\right)
\frac{\beta_u}{m-\beta_g}
>
\frac{\beta_g}{s}.
\end{equation}
If this active constraint is mistakenly dropped from the predicted pool, it can only be selected by the uniform branch:
\begin{equation}
p_{\mathrm{drop}}
=
\frac{\beta_u}{m-\beta_g}.
\end{equation}
Thus, adding $\Delta$ inactive constraints mainly dilutes the guided sampling probability from $\beta_g/s$ to $\beta_g/(s+\Delta)$. By contrast, dropping a true active constraint moves it from the guided branch to the uniform branch, reducing its sampling probability from approximately $\beta_g/s$ to $\beta_u/(m-\beta_g)$. Since in our setting the predicted active-set pool is much smaller than the full constraint set and the guided sampling ratio satisfies $\rho>1/2$, we have
\begin{equation}
\frac{\beta_u/(m-\beta_g)}{\beta_g/(s+\Delta)}
=
\frac{1-\rho}{\rho}
\cdot
\frac{s+\Delta}{m-\beta_g}
\ll 1.
\end{equation}
This explains why dropping true active constraints is more damaging than adding inactive constraints. At the same time, the perturbation results show that AT-SKM-Net remains robust under mild active-set prediction errors, and that the guided sampling strategy continues to provide acceleration even when the predicted pool contains moderate false positives.

\subsection{Transfer Learning and Pseudo-Label Active-Set Training}
\label{appendix:transfer_pseudo_label}

Although the experiments in the main paper use supervised training with solver-generated reference solutions, AT-SKM-Net does not fundamentally require exact optimal solutions for all training instances. The framework relies on two acceleration components: a warm-start candidate for equality projection and SKM iteration, and active-set scores for guided inequality sampling. These components can be trained or adapted with weaker supervision.

For the warm-start branch, prior work on T-SKM-Net has shown that the predictor can be trained without solver-generated ground-truth solutions by directly optimizing the objective together with constraint-violation penalties. For the active-set prediction branch, exact solver-derived active-set labels are also not strictly necessary. Since this branch is only used to bias the SKM sampling distribution toward likely active constraints, pseudo labels can be constructed from corrected solutions produced by the model itself. Specifically, for each target sample, the transferred model first produces a warm-start solution, the SKM layer then corrects it to a feasible solution, and the active constraints of this corrected solution are used as pseudo labels for retraining the active-set heads.

To evaluate this possibility, we consider two transfer settings: from the IEEE 57-bus system to the IEEE 118-bus system, and from the IEEE 118-bus system to the IEEE 300-bus system. In each setting, we first train a source model on the smaller system and transfer it to the larger target system. We then freeze the shared HGNN backbone and the warm-start output branch, and retrain only the active-set prediction heads using 5\% of the target $N$-1 scenarios. Importantly, these active-set heads are not trained with solver-derived active-set labels; instead, they use pseudo labels generated from the model's own corrected SKM outputs.

\begin{table}[h]
\centering
\scriptsize
\setlength{\tabcolsep}{5pt}
\caption{Transfer learning results with pseudo-label active-set training. Runtime values are reported in milliseconds. The notation $57{\rightarrow}118$ denotes transfer from the IEEE 57-bus system to the IEEE 118-bus system.}
\label{tab:transfer_pseudo_label}
\begin{sc}
\begin{tabular}{llccccccc}
\toprule
Transfer & Method & Eq. Proj. & Eq. Spd. & Ineq. Iter. & Ineq. Time & Ineq. Spd. & SKM Spd. & Gap (\%) \\
\midrule
$57{\rightarrow}118$ & T-SKM-Net &
$1.54 \pm 0.02$ & $1.00\times$ &
$57.31\,(148)$ & $5.93 \pm 7.07$ & $1.00\times$ & $1.00\times$ & $0.009$ \\
$57{\rightarrow}118$ & AT-SKM-AC &
$0.28 \pm 0.02$ & $5.50\times$ &
$29.28\,(59)$ & $1.96 \pm 1.58$ & $3.02\times$ & $3.33\times$ & $0.009$ \\
\midrule
$118{\rightarrow}300$ & T-SKM-Net &
$7.01 \pm 0.08$ & $1.00\times$ &
$30.83\,(218)$ & $2.06 \pm 2.53$ & $1.00\times$ & $1.00\times$ & $0.618$ \\
$118{\rightarrow}300$ & AT-SKM-AC &
$0.74 \pm 0.05$ & $9.47\times$ &
$6.30\,(18)$ & $0.51 \pm 0.54$ & $4.03\times$ & $7.26\times$ & $0.590$ \\
\bottomrule
\end{tabular}
\end{sc}
\end{table}

As shown in Table~\ref{tab:transfer_pseudo_label}, AT-SKM-AC retains substantial acceleration even when the active-set heads are trained only with pseudo labels. In the $57{\rightarrow}118$ transfer setting, the average number of inequality iterations decreases from $57.31$ to $29.28$, and the worst-case iteration count decreases from $148$ to $59$. In the more challenging $118{\rightarrow}300$ setting, the average iteration count decreases from $30.83$ to $6.30$, and the worst-case iteration count decreases from $218$ to $18$. These reductions show that pseudo-label active-set learning still provides useful sampling guidance after transfer to a larger system.

The runtime results show the same trend. In the $57{\rightarrow}118$ setting, AT-SKM-AC achieves a $5.50\times$ speedup in equality projection, a $3.02\times$ speedup in inequality iteration, and a $3.33\times$ total SKM speedup. In the $118{\rightarrow}300$ setting, the corresponding speedups are $9.47\times$, $4.03\times$, and $7.26\times$. Meanwhile, the optimality gap remains comparable to T-SKM-Net: $0.009\%$ versus $0.009\%$ for $57{\rightarrow}118$, and $0.590\%$ versus $0.618\%$ for $118{\rightarrow}300$. Since the final outputs are corrected by the projection and SKM layers, both methods preserve strict constraint feasibility after correction.

These results suggest that supervised ground-truth solutions are a training choice in the current implementation rather than a fundamental requirement of AT-SKM-Net. When exact solver-generated labels are expensive or insufficient on larger systems, the model can still transfer from a smaller system and use self-generated pseudo labels to train the active-set predictor, thereby retaining much of the acceleration benefit while preserving strict feasibility.

\end{document}